%% file: iclr2026_conference.tex
\documentclass{article} 
\usepackage{iclr2026_conference,times}

\input{math_commands.tex}

\usepackage{url}

\usepackage{amsmath}
\usepackage[table]{xcolor}
\usepackage{multirow}
\usepackage{enumitem}
\usepackage{graphicx}
\usepackage{amsthm}  
\usepackage{amsmath} 
\usepackage{amssymb} 
\usepackage{caption} 
\usepackage{makecell}
\usepackage{dsfont}
\usepackage{booktabs}
\usepackage{subcaption}
\usepackage{titletoc}                
\usepackage[dvipsnames,svgnames,x11names]{xcolor} 
\usepackage[colorlinks=true]{hyperref}

\usepackage[most]{tcolorbox}  
\usepackage{tabularx}       
\usepackage{array}          
\newtheorem{proposition}{Proposition}

\iclrfinalcopy

\title{LIHE: Linguistic Instance-Split Hyperbolic-Euclidean Framework for Generalized Weakly-Supervised Referring Expression Comprehension}

\author{%
  Xianglong Shi$^{1}$\thanks{Equal contribution.}, Silin Cheng$^{2}$, Sirui Zhao$^{1}$, Yunhan Jiang$^{4}$,\\
  \textbf{Enhong Chen$^{1}$, Yang Liu$^{3}$\thanks{Corresponding author.}, Sebastien Ourselin$^{3}$} \\
  $^{1}$University of Science and Technology of China, 
  $^{2}$The University of Hong Kong  \\
  $^{3}$King’s College London, $^{4}$Peking University\\
  \texttt{xlshi@mail.ustc.edu.cn}, \texttt{hnslcheng@connect.hku.hk},\texttt{yang.9.liu@kcl.ac.uk}\\
}

\begin{document}

\maketitle
\begin{abstract}

Existing Weakly-Supervised Referring Expression Comprehension (WREC) methods, while effective, are fundamentally limited by a one-to-one mapping assumption, hindering their ability to handle expressions corresponding to zero or multiple targets in realistic scenarios. To bridge this gap, we introduce the Weakly-Supervised Generalized Referring Expression Comprehension task (WGREC), a more practical paradigm that handles expressions with variable numbers of referents. However, extending WREC to WGREC presents two fundamental challenges: supervisory signal ambiguity, where weak image-level supervision is insufficient for training a model to infer the correct number and identity of referents, and semantic representation collapse, where standard Euclidean similarity forces hierarchically-related concepts into non-discriminative clusters, blurring categorical boundaries. To tackle these challenges, we propose a novel WGREC framework named Linguistic Instance-Split Hyperbolic-Euclidean (LIHE), which operates in two stages. The first stage, Referential Decoupling, predicts the number of target objects and decomposes the complex expression into simpler sub-expressions. The second stage, Referent Grounding, then localizes these sub-expressions using HEMix, our innovative hybrid similarity module that synergistically combines the precise alignment capabilities of Euclidean proximity with the hierarchical modeling strengths of hyperbolic geometry. This hybrid approach effectively prevents semantic collapse while preserving fine-grained distinctions between related concepts. Extensive experiments demonstrate LIHE establishes the first effective weakly supervised WGREC baseline on gRefCOCO and Ref-ZOM, while HEMix achieves consistent improvements on standard REC benchmarks, improving IoU@0.5 by up to 2.5\%. The code is available at \url{https://anonymous.4open.science/r/LIHE}.
\end{abstract}


\input{iclr2026/sections/1_introduction}
\input{iclr2026/sections/2_related_work}
\input{iclr2026/sections/3_problem}
\input{iclr2026/sections/4_Method}

\input{iclr2026/sections/5_expriments}
\input{iclr2026/sections/6_conclusion}

\bibliography{iclr2026_conference}
\bibliographystyle{iclr2026_conference}

\input{iclr2026/sections/appendix}

\end{document}

%% file: math_commands.tex
\usepackage{amsmath,amsfonts,bm}

\def\eqref#1{equation~\ref{#1}}

\def\1{\bm{1}}

\DeclareMathAlphabet{\mathsfit}{\encodingdefault}{\sfdefault}{m}{sl}
\SetMathAlphabet{\mathsfit}{bold}{\encodingdefault}{\sfdefault}{bx}{n}



%% file: iclr2026/sections/1_introduction.tex
\section{Introduction}

Referring Expression Comprehension (REC)~\citep{mao2016generation,yu2016modeling}, also known as visual grounding, aims to localize objects in an image based on natural language expressions. REC has shown broad application potential in fields such as robotic navigation and image editing. However, existing REC methods heavily rely on instance-level annotations, which are expensive and labor-intensive to collect~\citep{zhu2022seqtr,deng2023transvg}. To alleviate this bottleneck, Weakly-Supervised REC (WREC) has emerged as a cost-effective alternative. WREC methods eliminate the need for bounding box supervision and instead learn to align vision and language features using only image–text pairs. Early approaches~\citep{gupta2020contrastive, liu2019adaptive} typically employed two-stage pipelines, while recent single-stage frameworks~\citep{jin2023refclip,luo2024apl,cheng2025weakmcn} have become dominant due to their superior efficiency, directly matching text with anchor features from pre-trained detectors through contrastive learning~\citep{oord2018representation}.

 However, real-world scenarios are often more complex: a referring expression might correspond to multiple objects or no object at all, known as Generalized Referring Expression Comprehension (GREC)~\citep{he2023grec,liu2023gres,wu20243d}, which extends REC by allowing expressions to describe no-target and multi-target cases (shown as Fig.~\ref{fig:example}). Accordingly, in this paper, we aim to solve a new task, Weakly-Supervised Generalized Referring Expression Comprehension (WGREC).

Extending WREC to the WGREC setting poses two fundamental challenges that need to be addressed to develop effective weakly supervised methods. The first is the ambiguity of the supervisory signal, where the winner-takes-all mechanism of previous WREC methods is structurally incapable of locating all relevant targets, erroneously returning only a single instance, as shown in Fig.~\ref{fig:example}. The second is the semantic representation collapse, which arises because conventional contrastive learning methods rely on Euclidean similarity. This assumes flat, one-to-one alignments, leading to suboptimal representations in multi-referent scenarios. For instance, an expression like ``left person'' may refer to both ``left man'' and ``left woman'' as shown in Fig.~\ref{heratical}. Pulling both specific instances toward the same general anchor (“person”) in Euclidean space unintentionally forces ``man'' and ``woman'' closer together, blurring their categorical boundaries.


\begin{figure*}
    \centering
    \vspace{-4mm}
    \includegraphics[width=1.0\linewidth]{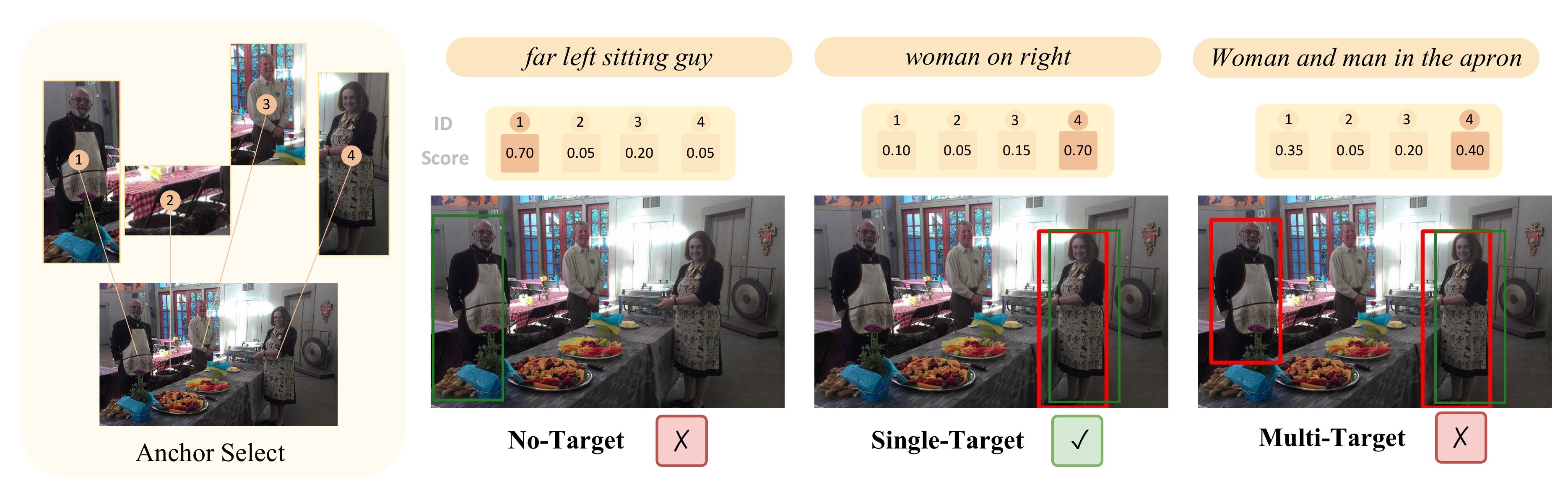}
    \caption{Limitations of current WREC methods. The ground truth is denoted by red bounding boxes, whereas green bounding boxes denote the predictions. Current WREC methods always select only the best anchor as output, failing to handle No-target and Multi-target cases (e.g., no red bounding box and two red bounding boxes).
    }
    \label{fig:example}
    \vspace{-5mm}
\end{figure*}

To address these challenges, we propose a two-stage \textbf{L}inguistic \textbf{I}nstance-Split \textbf{H}yperbolic-\textbf{E}uclidean (LIHE) framework designed to effectively localize a variable number of objects from a single expression. First, in the \textbf{Referential Decoupling} stage, LIHE leverage a vision-language model (VLM)~\citep{Qwen-VL,Qwen2VL,qwen2.5-VL} to infer the number of referring objects: if zero, it will skip the second stage; otherwise, it will decompose the expression into corresponding sub-expressions for each object. Second, the \textbf{Referent Grounding} stage employs a contrastive learning paradigm to train the model to localize each sub-expression. To prevent the semantic collapse inherent in this grounding process, LIHE integrates \textit{HEMix}, a hybrid similarity module that replaces the standard Euclidean metric. HEMix synergizes the fine-grained alignment of Euclidean proximity with the structure-preserving properties of hyperbolic geometry~\citep{ganea2018hyperbolic,hier,bdeir2024fully}. It leverages the inherent capacity of Lorentzian hyperbolic space to model hierarchies by placing general concepts near the origin and specific instances outward, thereby preserving semantic distinctions. This design, which serves as a \textit{plug-and-play} module, effectively prevents concept collapse and enhances generalization in complex referring scenarios with negligible computational overhead.

We conduct extensive experiments on the GREC datasets gRefCOCO~\citep{he2023grec} and Ref-ZOM~\citep{hu2023beyond}, demonstrating that our method is the first weakly supervised framework to tackle the WGREC task effectively. Additionally, we validate the broader applicability of HEMix on REC benchmarks, including RefCOCO~\citep{nagaraja2016modeling}, RefCOCO+~\citep{nagaraja2016modeling}, and RefCOCOg~\citep{mao2016generation}, and achieve the state-of-the-art performance. Our findings highlight the potential of combining REC with structured geometry to advance vision-language understanding.

%% file: iclr2026/sections/3_problem.tex
\vspace{-2mm}

\section{Task Definition}

\vspace{-3mm}


The Weakly-Supervised Generalized Referring Expression Comprehension (WGREC) task aims to localize all image regions described by a given natural language expression using only weak supervision. Formally, given an image $I$ and an expression $T$, the goal is to predict a set of bounding boxes, $\mathcal{B}^* = \{b_i \mid i \in k\}$, where each box $b_i$ corresponds to a region in $I$ that matches the expression $T$. and the number of targets $k$ is variable (zero, one, or more). Critically, no bounding box annotations are used during training. Each sample consists only of an image--text pair $(I, T)$, without knowing which regions match the expression. Unlike WREC, which assumes every expression refers to a specific object, WGREC allows for expressions that have no matching region. To address this, we introduce a binary label $v \in \{0,1\}$ for the training set, where $v = 0$ means no region in the image matches the expression, and $v = 1$ otherwise. This provides weak supervision to guide the model in learning to predict all matching regions. Our goal is to train a model that, using only these weak labels, can predict the complete set of referent boxes for inference.

\vspace{1em} 

%% file: iclr2026/sections/4_Method.tex
\vspace{-3mm}
\section{Method}
\vspace{-2mm}

Extending WREC to WGREC is non-trivial and typically faces two major challenges: \textbf{(1) Cardinality Ambiguity}:  WREC methods such as RefCLIP simplify the task by reducing it to an anchor–text matching problem. Specifically, they select the most relevant anchor from a predefined set $\mathcal{A}$ using:
\begin{equation}
a^* = \arg\max_{a \in \mathcal{A}} \phi(T, I, a),
\label{RECTask}
\end{equation}
where $\phi(T, I, a)$ denotes the similarity between the expression $T$, the image $I$, and anchor $a$. However, this max-selection strategy inherently assumes a single referent, making it unsuitable for WGREC, where the number of targets is unknown. \textbf{(2) Hierarchical Representation
Collapse}: When a general expression $T$ (e.g., ``person'') refers to multiple distinct sub-categories (e.g., a ``man'' and a ``woman'') as shown in Fig.~\ref{heratical}, conventional contrastive learning in Euclidean space can conflate their representations. This blurs categorical boundaries and leads to a loss of semantic distinction.

To address these limitations, we propose \textbf{L}inguistic \textbf{I}nstance-Split \textbf{H}yperbolic-\textbf{E}uclidean (LIHE), a framework designed for WGREC. As shown in Fig.~\ref{fig:network}, LIHE consists of a \textit{Referential Decoupling} stage, which decomposes complex expressions into single-instance queries, and a \textit{Referent Grounding} stage, which detects all matching regions. In addition, we introduce a \textit{HEMix} similarity to explicitly preserve hierarchical relationships.

\vspace{-1mm}
\subsection{Referential Decoupling}
\vspace{-1mm}
In the context of WGREC, referring expressions often correspond to multiple visual entities. To address this, we reformulate the task by decomposing a multi-target expression into a set of single-target sub-expressions, allowing each referent to be localized independently and utilize the capabilities of VLM~\citep{qwen2.5-VL} to judge whether the target exists. This strategy simplifies the multi-instance grounding problem and makes it more tractable under weak supervision.

We leverage the perceptual capabilities of large vision-language models (VLMs)~\citep{qwen2.5-VL}, which, although do not have grounding functions, exhibit strong visual understanding and language reasoning. Given an image $I$, a referring expression $T$, and a carefully designed prompt $P$, we input them into a VLM to obtain a set of simplified, instance-level expressions.
To exploit the in-context learning capability of the VLM, we design a four-part prompt paired with the input image to guide the VLM in decomposing the original referring expression:
(1) General instruction $\mathbf{P_G}$, describing the goal of splitting the expression into target-specific phrases;
(2) Output constraints $\mathbf{P_C}$, specifying the format of each phrase;
(3) In-context examples $\mathbf{P_E}$, providing annotated demonstrations to steer the VLM toward the desired behavior;
(4) Input query $\mathbf{P_Q}$, which contains original referring expression together with explicit instructions;
(5) Raw Image $\mathbf{I}$.
Decomposition is formulated as
\begin{equation}
K,\mathcal{T}_{\text{D}} = \text{VLM}(\mathbf{P_G},\, \mathbf{P_C},\, \mathbf{P_E},\, \mathbf{P_Q},\, \mathbf{I}),
\end{equation}
where $K$ is the number of targets and $\mathcal{T}_{\text{D}} = \{ t_1, t_2, \dots, t_k \}$ is the set of sub-expressions generated by the model, each describing a distinct visual entity potentially present in the image. The prompts $\mathbf{P}$ guide the model to rewrite the original expression into concise, non-overlapping descriptions of individual targets, using brief instructions and a few examples. This decomposition mitigates issues such as irrelevance and ambiguity. 
Meanwhile, the prompt component $\mathbf{P_C}$ explicitly instructs the VLM to first output the number of target phrases \(K\) before listing them. 
This constraint helps mitigate common hallucination issues in VLMs, such as generating duplicate referring expressions for the same visual entity. 
When the VLM returns ($K = 0$), we interpret it as a no-target case, which naturally fits the open-ended setting of WGREC. More detailed prompt design, please see Appendix~\ref{Appendix:DetailedPromptDesign}.

\begin{figure*}
    \centering
    \vspace{-6mm}
    \includegraphics[width=1.0\linewidth]{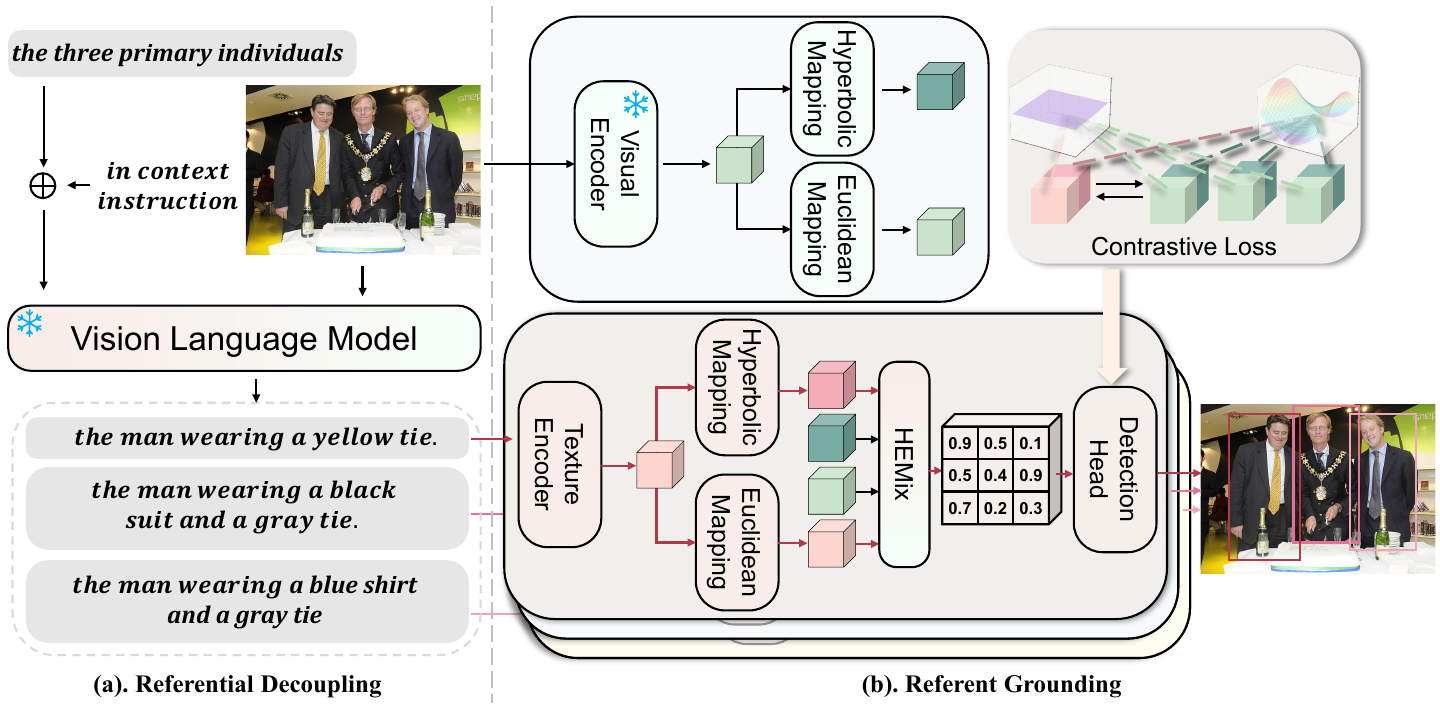}
    \caption{\textbf{The overall framework of LIHE}. (a). Referential Decoupling: VLM decomposed the referring expression into distinct short phrases for each target. (b). Referent Grounding: Each phrase is processed by a textual encoder, and the image by a visual encoder. Then the model filters anchors of low value and returns the best-matching one for bounding box prediction. The referent grounding stage is weakly supervised by the anchor-based contrastive loss.}
    \label{fig:network}
    \vspace{-5mm}
\end{figure*}

\subsection{Referent Grounding}
After the decoupling stage, each decomposed referring expression \(t\) corresponds to a matched visual entity in the image \(I\). Consequently, the task reduces to a conventional WREC problem, as defined in Eq.~\ref{RECTask}. Therefore, the referent grounding stage of our framework follows the structure and training strategy of previous WREC methods~\citep{jin2023refclip,luo2024apl,cheng2025weakmcn}, which have proven effective for WREC. Specifically, given an input image and a referring expression (the short phrases generated from the referential decoupling stage), the model uses a one-stage detector to extract visual feature maps, from which anchor features are obtained.
We retain only the anchors from the last feature map layer—based on the assumption that most target objects in referring datasets are of moderate to large size and further filter them by confidence, typically keeping the top 10\% of anchors.
Each remaining anchor is projected into a joint semantic space, alongside the corresponding text embedding. The similarity between each anchor and the phrase is then computed using a hybrid similarity metric, which combines Euclidean and Hyperbolic similarity scores. 
Formally, we adopt the contrastive learning objective of RefCLIP~\citep{jin2023refclip}, replacing its similarity function with our HEMix:
\begin{equation}
\label{constrastive}
\mathcal{L}_c = - \log \frac{\exp\big( \text{HEMix}(f_{a_0}^i, f_t^i)/\tau \big)}{
\sum\limits_{n=0}^{N} \sum\limits_{j=0}^{M} \mathbb{I}_{\neg (i=j \wedge n \neq 0)} \exp\big( \text{HEMix}(f_{a_n}^j, f_t^i)/\tau \big)}
\end{equation}
where $f_{a_0}^i$ is the positive anchor for $i$-th image, $f_t^i$ is the text embedding, $\tau$ is a temperature scalar, and HEMix denotes our proposed Euclidean-Hyperbolic hybrid similarity. This contrastive objective aligns the correct anchor-text pair while using both intra- and inter-image anchors as negatives. More details of HEMix formulation are provided in Sec.~\ref{sec:hybrid-sim}. Note that we filter the training dataset for the referent grounding stage by retaining only samples with validity flag \(v=1\) (i.e., at least one entity in \(I\) matches \(T\)). 

\subsection{HEMix}
\label{sec:hybrid-sim}

In previous WREC methods~\citep{jin2023refclip,luo2024apl,cheng2025weakmcn}, contrastive learning is typically driven by Euclidean similarity in a shared embedding space. However, this approach is limited in its ability to capture hierarchical semantics. For instance, as illustrated in Fig.~\ref{heratical}, a phrase like \textit{`left person'} may refer to both \textit{`left man'} and \textit{`left woman'} two visual entities in the image, which are semantically related but visually distinct. Euclidean similarity tends to cluster all such instances together, resulting in ambiguous localization.
\begin{figure*}
    \centering
    \setlength{\abovecaptionskip}{6pt}   
    \setlength{\belowcaptionskip}{1pt}  
    \includegraphics[width=1.0\linewidth]{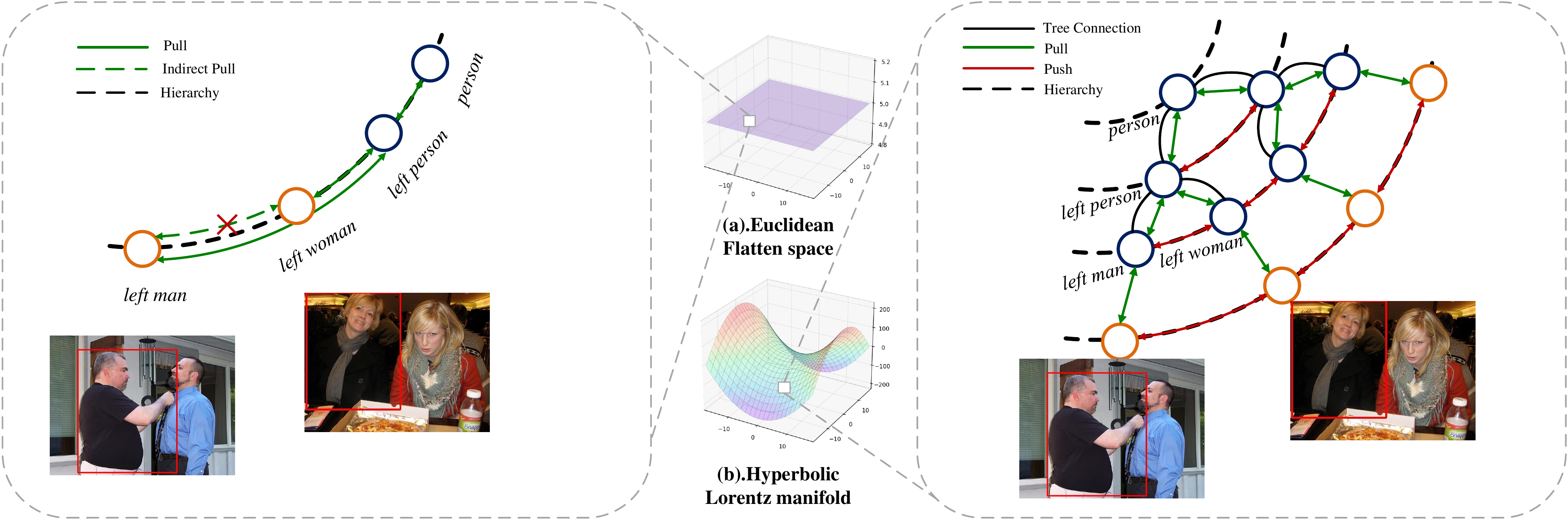}
    \caption{A simple illustration of (a) Euclidean flatten space and (b) hyperbolic Lorentz manifold in 3-dimensional space~\citep{li2024hyperbolic}. In Euclidean space, all nodes occupy a single, undifferentiated hierarchy, so parent and child entities share the same geometric scale.
    In contrast, the negative curvature of hyperbolic space naturally organizes nodes into concentric hierarchies: parent nodes reside closer to the manifold’s apex, while their children are pushed farther outward, and different children at the same level are repelled from one another.}
    \label{heratical}
\end{figure*}
To better model semantic structure, we incorporate hyperbolic geometry using the Lorentz (hyperboloid) model. Hyperbolic spaces naturally embed hierarchical relationships: pulling a parent concept (e.g., \textit{`person'}) closer to the time axis increases angular separation among its children (e.g., \textit{`man'} and \textit{`woman'}), effectively preserving both generality and specificity. This aligns well with the hierarchical nature of referring expressions. More introduction and insights are in Appendix~\ref{Appendix:Hyperbolic Space}. 

For $D$-Dimension visual feature $f_v \in \mathbb{R}^{D}$ and text feature $f_t \in \mathbb{R}^{D}$, we use two types of similarity:

(1) \textbf{Euclidean similarity}: Same as refclip~\citep{jin2023refclip}, the similarity is calculated by
\begin{equation}
\text{Sim}_{\mathrm{E}}(f_v, f_t) = \langle f_v \mathbf{W}_{EV},\; f_t \mathbf{W}_{ET} \rangle,
\end{equation}
where $\mathbf{W}_{EV}$ and $\mathbf{W}_{ET}$ are learnable linear mapping matrices, and $\langle \cdot, \cdot \rangle$ denotes the standard inner product in Euclidean space.

(2) \textbf{Hyperbolic similarity (Lorentz model)}: We first map features into hyperbolic space as spatial components of \(\tilde{f}_{v}, \tilde{f}_{t}\) in hyperbolic space:
\begin{equation}
\mathbf{z}_v = f_v \mathbf{W}_{HV} \in \mathbb{R}^{D}, \quad \mathbf{z}_t = f_t \mathbf{W}_{HT} \in \mathbb{R}^{D},
\end{equation}
where $\mathbf{W}_{HV}$ and $\mathbf{W}_{HT}$ are learnable linear projection matrices for hyperbolic space.
Here, instead of using exponential maps to embed into hyperbolic space in the paper before, which is prone to unstable gradients, we adopt a learnable linear projection for both branches. This ensures smooth training and better compatibility with contrastive objectives. 
In a hyperbolic space of curvature \(\kappa\), the calculation of hyperbolic similarity by the Lorentzian inner product is formulated as: 
\begin{equation}
\text{Sim}_{\mathrm{H}}(f_v, f_t) = \langle \tilde{f}_v, \tilde{f}_t \rangle_{\mathbb{H}} = -x_0^v x_0^t + \langle \mathbf{z}_v, \mathbf{z}_t \rangle,
\end{equation}
where the mapped feature vectors \(\tilde{f}_v = (x_0^v, \mathbf{z}_v) \in \mathbb{R}^{D+1},\quad
\tilde{f}_t = (x_0^t, \mathbf{z}_t) \in \mathbb{R}^{D+1}\) and the time components
\(x_0 = \smash[t]{\sqrt{\lVert \mathbf{z}\rVert^{2} + \kappa^{-1}}}\).
Due to the time component \(x_0\) being calculated from the spatial
component \(\mathbf{z}\), the feature vectors \(\tilde{f}\) satisfy
\(\langle \tilde{f}, \tilde{f} \rangle_{\mathbb{H}} = -\kappa^{-1}\).
Thus \(\tilde{f}\) represents a valid point in
\(
\mathbb{H}^n_\kappa = \{\mathbf{x} \in \mathbb{R}^{n+1} \mid
\langle \mathbf{x}, \mathbf{x} \rangle_\mathbb{H} = -\kappa^{-1},\; x_0 > 0\}
\)
and inherits the geometric properties of hyperbolic space.









(3) \textbf{HEMix similarity}: We define the final similarity as a weighted combination of the two:
\begin{equation}
\text{HEMix}(f_v, f_t) = (1 - \alpha) \, \text{Sim}_{\mathrm{E}}(f_v, f_t) + \alpha \, \text{Sim}_{\mathrm{H}}(f_v, f_t), \quad \alpha \in (0, 1).
\end{equation}
Euclidean space ($\kappa\!=\!0$ case) excels at \emph{local, flat} geometry; it preserves fine‑grained angular relationships that are crucial for pixel‑accurate localization.  
Lorentzian Hyperbolic space ($\kappa\!<\!0$ case) naturally embeds \emph{hierarchical} or long‑range semantics because geodesic distance grows exponentially~\citep{hyplorentz,ganea2018hyperbolic}, but its metric stretches neighborhoods close to the light cone, which can blur local details.  
Each geometry similarity (\(\text{Sim}_{\mathrm{E}}\) and \(\text{Sim}_{\mathrm{H}}\)) therefore introduces a different \emph{estimation error} with respect to the ideal but unknown similarity $\text{Sim}^{\star}$:
\[
\text{Sim}_{\mathrm{E}}=\text{Sim}^{\star}+b_{\mathrm{E}}+\varepsilon_{\mathrm{E}},\quad
\text{Sim}_{\mathrm{H}}=\text{Sim}^{\star}+b_{\mathrm{H}}+\varepsilon_{\mathrm{H}},
\]
where $b =\mathbb E[\text{Sim} -\text{Sim}^{\star}]$ denotes the \emph{bias} and $\varepsilon$ the zero‑mean random deviation. Specifically, $b_{\mathrm{E}}$ is large for hierarchical descriptions, while $b_{\mathrm{H}}$ is large for micro‑spatial references. The two errors are not perfectly correlated, due to the characteristics of each space.

\begin{proposition}[Variance reduction]
\label{prop:var-reduce}
Let $\sigma^{2}_{\mathrm{E}}\!=\!\operatorname{Var}[\varepsilon_{\mathrm{E}}]$,
$\sigma^{2}_{\mathrm{H}}\!=\!\operatorname{Var}[\varepsilon_{\mathrm{H}}]$
and $\rho\!=\!\operatorname{Corr}[\varepsilon_{\mathrm{E}},\varepsilon_{\mathrm{H}}]$.
If $\rho<1$, the mean‑squared error of the hybrid estimator
\begin{align*}
\operatorname{MSE}(\text{HEMix})
&=\mathbb E\bigl[(\text{HEMix}-\text{Sim}^{\star})^{2}\bigr]\\
&=((1-\alpha)b_{\mathrm{E}}+\alpha b_{\mathrm{H}})^{2}
\quad +(1-\alpha)^{2}\sigma_{\mathrm{E}}^{2}+\alpha^{2}\sigma_{\mathrm{H}}^{2}+2\alpha(1-\alpha)\rho\sigma_{\mathrm{E}}\sigma_{\mathrm{H}},
\end{align*}
\text{attains its minimum at}
$\displaystyle
\alpha^{\star}=\frac{(\sigma_{\mathrm{E}}^{2}+\rho\sigma_{\mathrm{E}}\sigma_{\mathrm{H}})+b_{\mathrm{E}}(b_{\mathrm{E}}-b_{\mathrm{H}})}{\sigma_{\mathrm{E}}^{2}+\sigma_{\mathrm{H}}^{2}-2\rho\sigma_{\mathrm{E}}\sigma_{\mathrm{H}}+(b_{\mathrm{E}}-b_{\mathrm{H}})^{2}},
$
\text{and make} $\operatorname{MSE}(\text{HEMix})$ \text{satisfies}
\begin{equation}
    \operatorname{MSE}(\text{HEMix};\alpha^{\star})<\min\bigl\{\operatorname{MSE}(\text{Sim}_{\mathrm{E}}),\,\operatorname{MSE}(\text{Sim}_{\mathrm{H}})\bigr\}.
\end{equation}
\end{proposition}
Please see the supplementary material for complete proof. Under the common assumption~\citep{huang2021towards,gouk2021regularisation,lei2023generalization} that the contrastive loss \(\mathcal{L}\!\bigl(\sigma(S)\bigr)\) is Lipschitz-continuous in its similarity argument \(S\), a lower MSE implies a tighter generalization bound, which translates into higher retrieval or localization accuracy. Due to \(\rho\) is unknown in practice, we select \(\alpha\) via experiment. Overall, HEMix unifies two complementary geometric inductive biases, local Euclidean precision and hyperbolic hierarchy, into a single estimator that better approximates \(\text{Sim}^{\star}\), as confirmed empirically in Sec.~\ref{Experiments}.

%% file: iclr2026/sections/5_expriments.tex
\vspace{-2mm}
\section{Experiments}
\label{Experiments}
\vspace{-1mm}
\subsection{Datasets and Metrics}
\vspace{-2mm}

We evaluate our proposed method on two benchmark datasets for the WGREC task: gRefCOCO~\citep{he2023grec} and Ref-ZOM~\citep{wang2025hierarchical}, both of which support expressions that may correspond to multiple or zero referents. Following prior works~\citep{he2023grec}, we adopt Precision@(F$_1$=1, IoU$\ge$0.5) and N-acc. as the main evaluation metrics. 
Precision@(F$_1$=1, IoU$\ge$0.5) computes the percentage of samples that have the F1 score of 1 with the IoU threshold set to 0.5 and N-acc. assesses the model’s proficiency in no-target identification.
Detailed explanations are provided in the Appendix~\ref{appendx:metrics}. 
To further validate the effectiveness and generalization ability of our proposed HEMix module, we also evaluate it on three widely-used WREC datasets: RefCOCO, RefCOCO+, and RefCOCOg. For these datasets, we follow the standard evaluation protocol using IoU@0.5, where a prediction is considered correct if the IoU between the predicted and ground-truth bounding box exceeds 0.5.

\begin{table}[t]
\setlength{\abovecaptionskip}{6pt}
\centering
\footnotesize
\caption{Comparison on the WGREC task. `$\dagger$' denotes these methods have been modified to generate multiple boxes following \citet{he2023grec}. 
`$^*$' denotes the model adapted for the WGREC task.}
\resizebox{\textwidth}{!}{
\setlength{\tabcolsep}{1.5mm} 
{\renewcommand{\arraystretch}{1.2}
\begin{tabular}{l|c|cc|cc|cc|cc}
\hline
\multirow{3}{*}{\textbf{Methods}} & \multirow{3}{*}{\textbf{Supervision}}  & \multicolumn{6}{c|}{\textbf{gRefCOCO}}& \multicolumn{2}{c}{\textbf{Ref-ZOM}} \\
\cline{3-8} \cline{9-10}
 &  &  \multicolumn{2}{c|}{val} & \multicolumn{2}{c|}{testA} & \multicolumn{2}{c|}{testB} & \multicolumn{2}{c}{val} \\
 & & Pr (\%) & N-acc.(\%) & Pr(\%) & N-acc.(\%) & Pr(\%) & N-acc.(\%) & Pr(\%) & N-acc.(\%) \\
\hline
MCN$^{\dag}$~\citep{luo2020multi}           & Fully & 28.0 & 30.6 & 32.3 & 32.0 & 26.8 & 30.3 & - & - \\
VLT$^{\dag}$~\citep{ding2021vision}           & Fully & 36.6 & 35.2 & 40.2 & 34.1 & 30.2 & 32.5 & - & -\\
MDETR$^{\dag}$~\citep{kamath2021mdetr}        & Fully & 42.7 & 36.3 & 50.0 & 34.5 & 36.5 & 31.0 & - & -\\
UNITEXT$^{\dag}$~\citep{yan2023universal}       & Fully  & 58.2 & 50.6 & 46.4 & 49.3 & 42.9 & 48.2 & - & -\\
Ferret-7B$^{*}$~\citep{you2023ferret}          & Fully & 54.8 & 48.9 & 49.5 & 45.2 & 43.5 & 43.8 & - & -\\
VistaLLM-7B~\citep{pramanick2024jack}          & Fully & 52.7 & 69.4 & - & - & - & - & - & -\\
VistaLLM-13B~\citep{pramanick2024jack}          & Fully  & 54.6 & 70.8 & - & - & - & - & - & -\\
RECANTFormer(5)~\citep{hemanthage2024recantformer} & Fully  & 57.73 &  52.70 & 57.82 & 53.38 & 49.49 & 54.53 & 56.69 &  88.24\\
RECANTFormer(10)~\citep{hemanthage2024recantformer}  & Fully &  55.10 & 52.73 & 55.07 & 53.07 & 48.01 & 54.81 & 59.78 &88.24\\
HieA2G$_\textrm{R101}$~\citep{wang2025hierarchical}  & Fully  & 67.8 & 60.3 & 66.0 & 60.1 & 56.5 & 56.0 & - & -\\
\hline
RefCLIP$^{*}$~\citep{jin2023refclip}           & Weakly &  17.85 & 0.0 & 18.23 & 0.0 & 21.89 & 0.0 & 35.78 & 0.0 \\
\rowcolor{gray!15}
\textbf{LIHE} & \textbf{Weakly}  & \textbf{39.61} & \textbf{67.49} & \textbf{32.70} & \textbf{79.60} & \textbf{35.84} & \textbf{67.07} & \textbf{50.36} & \textbf{97.70} \\
\hline
\end{tabular}}}

\label{GREC}
\end{table}

\vspace{-1mm}
\subsection{Comparisons with State-of-the-art Methods}
\vspace{-2mm}
\paragraph{WGREC results.} 
As shown in Tab.~\ref{GREC}, LIHE achieves strong performance on the gRefCOCO dataset under the weakly supervised setting. 
Compared to other WREC baselines, RefCLIP~\citep{jin2023refclip}, our model significantly outperforms it across all splits. For instance, on the validation set, our method achieves 39.61\% grounding precision and 67.49\% normalized accuracy, while RefCLIP only obtains 17.85\% and lacks the capability to handle no-target cases. 
This performance gap clearly demonstrates that methods assuming a single target fail to generalize to WGREC, which involves multi-target and no-target scenarios. 
Although fully supervised methods like HieA2G~\citep{wang2025hierarchical} and RECANTFormer~\citep{hemanthage2024recantformer} outperform our model in absolute metrics, our method remains competitive despite using only image-level supervision, and even surpasses earlier fully supervised baselines such as MCN~\citep{luo2020multi}, VLT~\citep{ding2021vision}, and MDETR~\citep{kamath2021mdetr}. 
Additionally, it is worth mentioning that our method can even accomplish unsupervised GREC tasks,
as shown in Appendix~\ref{Unsupervised Schema}.
\vspace{-3mm}

\begin{table}[t]
\centering
\footnotesize
\setlength{\tabcolsep}{2mm}
\caption{Performance of our methods on RefCOCO, RefCOCO+ and RefCOCOg datasets. `*' indicates results reproduced under identical settings.}
\resizebox{0.9\textwidth}{!}{
\renewcommand{\arraystretch}{1.2}
\begin{tabular}{l|c|ccc|ccc|c}
\hline
\multirow{2}{*}{\textbf{Method}} & 
\multirow{2}{*}{\textbf{Published on}} & 
\multicolumn{3}{c|}{\textbf{RefCOCO}} & 
\multicolumn{3}{c|}{\textbf{RefCOCO+}} & 
\multirow{1}{*}{\textbf{RefCOCOg}} \\
& & val & testA & testB & val & testA & testB & val \\
\hline
VC~\citep{niu2019variational}           & \textit{CVPR '18}    & -     & 32.68 & 27.22 & -     & 34.68 & 28.10 & 29.65 \\
ARN~\citep{liu2019adaptive}             & \textit{ICCV '19}    & 32.17 & 35.25 & 30.28 & 32.78 & 34.35 & 32.13 & 33.09 \\
IGN~\citep{zhang2020counterfactual}     & \textit{NeurIPS '20} & 34.78 & -     & -     & 36.91 & 36.91 & 35.46 & 34.92 \\
DTWREG~\citep{sun2021discriminative}    & \textit{TPAMI '21}   & 38.35 & 39.51 & 37.01 & 38.19 & 39.91 & 37.09 & 42.54 \\
\hline
RefCLIP*~\citep{jin2023refclip}   & \textit{CVPR '23}    & 59.88 & 58.44 & 56.91 & 40.11 & 40.01 & 38.63 & 47.87 \\
\textbf{RefCLIP*+HEMix} & - & \textbf{60.95} & \textbf{59.84} & \textbf{58.57} & \textbf{41.48} & \textbf{42.54}& \textbf{39.37} & \textbf{48.67} \\
\hline
APL*~\citep{luo2024apl}    & \textit{ECCV '24}    & 64.18 &  61.06 & 63.08 & 41.03 & 41.46 & 38.72 & 49.45 \\
\textbf{APL*+HEMix}  & -  & \textbf{65.71} & \textbf{62.67} & \textbf{64.04} & \textbf{42.13} & \textbf{42.98} & \textbf{40.70} & \textbf{50.88} \\
\hline
WeakMCN~\citep{cheng2025weakmcn}    & \textit{CVPR '25}    & 69.20 &  69.88 & 62.63 & 51.90 & 57.33 & 43.10 & 54.62 \\
\textbf{WeakMCN*+HEMix}  & -  & \textbf{70.44} & \textbf{71.59} & \textbf{63.22} & \textbf{52.61} & \textbf{58.14} & \textbf{44.43} & \textbf{55.60} \\
\hline
\end{tabular}
\label{tab:wrec-results}
}
\end{table}

\paragraph{WREC results.} 
We further evaluate the generalization ability of our hybrid similarity learning scheme on standard WREC benchmarks, including RefCOCO, RefCOCO+, and RefCOCOg. 
As shown in Tab.~\ref{tab:wrec-results}, integrating our proposed hybrid similarity module consistently improves existing baselines. 
For example, RefCLIP+HEMix surpasses vanilla RefCLIP on all splits (e.g., +1.71\% on RefCOCO testA, +1.15\% on RefCOCO+testA). 
Similarly, APL+HEMix yields consistent gains over APL~\citep{luo2024apl}, improving RefCOCOg by 1.43\% and RefCOCO+ testB by 1.98\% and WeakMCN+HEMix also improves the performance on three datasets. 
These results highlight the compatibility and plug-and-play nature of our hybrid similarity formulation. 
By combining Euclidean and Hyperbolic similarity measures, the model benefits from both local discriminative alignment and global semantic consistency, leading to better grounding performance across varying datasets and expression types.
\vspace{-2mm}
\subsection{Ablation Study}

\begin{table}[t]
\noindent
\footnotesize
\renewcommand{\arraystretch}{1.2}
\makebox[\textwidth][l]{ 

\begin{minipage}[t]{0.5\textwidth}

\caption{Performance of different similarity methods in contrastive loss.}
\centering
\setlength{\tabcolsep}{1.2mm}
\resizebox{\linewidth}{!}{
\begin{tabular}{c|cc|ccc}
\hline
\multirow{2}{*}{\textbf{Similarity Method}} &\multicolumn{2}{c|}{\textbf{RefCOCO (WREC)}} &\multicolumn{3}{c}{\textbf{gRefCOCO (WGREC)}}\\

& testA & testB &val &testA &testB \\
\hline
$\mathrm{Sim}_\mathrm{E}$& 58.44  & 56.92 & 38.88	& 31.77	 &	34.89 \\
$\mathrm{Sim}_\mathrm{H}$& 58.94  & 57.72  & 39.58	 & 31.57  &	 36.88\\

 HEMix     & \textbf{59.84}  & \textbf{58.57} & \textbf{39.73}  & \textbf{32.19}  &  \textbf{37.22}\\

\hline
\end{tabular}
}
\label{ablation:sim}
\end{minipage}

\hspace{0.05\textwidth}




\begin{minipage}[t]{0.35\textwidth}
\centering
\renewcommand{\arraystretch}{1}
\setlength{\tabcolsep}{1.2mm}
\caption{Ablation study on the prompt design with N-acc. metrics}

\resizebox{0.95\linewidth}{!}{
\begin{tabular}{c|ccc}
\toprule
\multirow{2}{*}{\textbf{Prompt design}} & \multicolumn{3}{c}{\textbf{gRefCOCO (WGREC)}} \\
              & val & testA & testB \\
\midrule
Without $P_E$ & 49.00 & 65.60 & 50.93 \\
$P^*$         & 68.10 & 76.66 & 68.26 \\
$P$           & 67.49 & 79.60 & 67.07 \\
\bottomrule
\end{tabular}
}
\label{ablation:prompt}
\end{minipage}
\vspace{-5mm}

}
\end{table}
\begin{table}[t]
\centering
\footnotesize
\caption{Cross-Dataset validation from GREC to REC benchmarks. In-Dataset denotes the model trained on the same dataset as the test.}
\vspace{-3mm}
\resizebox{0.75 \linewidth}{!}{
\renewcommand{\arraystretch}{1.2}
\begin{tabular}{l|l|ccc|ccc|c}
\hline
\multirow{2}{*}{\textbf{Method}} & 
\multirow{2}{*}{\textbf{Training Set}} & 
\multicolumn{3}{c|}{\textbf{RefCOCO}} & 
\multicolumn{3}{c|}{\textbf{RefCOCO+}} & 
\multirow{1}{*}{\textbf{RefCOCOg}} \\
& & val & testA & testB & val & testA & testB & val \\
\hline
RefCLIP & \multirow{2}{*}{In-Dataset} &59.88 &58.44 &56.91 &40.11 &40.01 &38.63 &47.87 \\
Ours &            & \textbf{60.95} & \textbf{59.84} & \textbf{58.57} & 41.48 & 42.54 & \textbf{39.37}  & 48.67     \\
\hline
RefCLIP    & \multirow{2}{*}{gRefCOCO}       & 47.62 & 46.83 & 48.53 & 38.25 & 38.91 & 38.00 & 43.73 \\
Ours    & & 54.34 & 55.79 & 51.38 & \textbf{42.04} & \textbf{43.22} & 38.84 & \textbf{49.14} \\
\hline
\end{tabular}
}
\label{ablation3}
\end{table}
\vspace{-2mm}
\begin{table}
\caption{Performance of different hyperbolic mapping methods on RefCOCO and WGREC.}
\vspace{-3mm}
\centering
\renewcommand{\arraystretch}{1.2}

\resizebox{0.75 \linewidth}{!}{
\begin{tabular}{c|ccc|ccc}
\hline
\multirow{2}{*}{Hyperbolic Mapping Method} & \multicolumn{3}{c|}{RefCOCO} & \multicolumn{3}{c}{WGREC}\\
& val & testA & testB & val & testA & testB\\
\hline
Exponential Map    & 55.02 & 54.16 & 53.78 & 38.65 & 31.44 & 34.37 \\
Learnable Linear   & \textbf{60.95} & \textbf{59.84} & \textbf{58.57} & \textbf{39.61} & \textbf{32.70} & \textbf{35.84} \\
\hline
\end{tabular}
}
\vspace{-3mm}
\label{ablation:mapall}
\end{table}

To validate the effectiveness of individual components in our framework, we conduct comprehensive ablation studies covering similarity modules, mapping strategies and cross-dataset validation. 
\vspace{-3mm}
\paragraph{Effect of Similarity Design in Contrastive Learning.}
We conduct an ablation study to assess the impact of different similarity functions in contrastive learning, as shown in Tab.~\ref{ablation:sim}. Overall, our proposed HEMix consistently outperforms standard Euclidean similarity ($\mathrm{Sim}_\mathrm{E}$, equivalent to Referent Grounding directly using RefCLIP), with average gains of 1.53\% on RefCOCO (WREC) and 0.90\% on gRefCOCO (WGREC), demonstrating its effectiveness across both tasks. When comparing $\mathrm{Sim}\mathrm{H}$ and $\mathrm{Sim}\mathrm{E}$, the improvement is more pronounced on WGREC than on WREC (+0.83\% vs. +0.65\%), and notably, $\mathrm{Sim}_\mathrm{H}$ achieves performance close to HEMix on WGREC, with only a 0.44\% gap on average. This suggests that hyperbolic similarity better supports the grounding of multiple semantically related referents, which is common in WGREC. One likely reason is that expressions in WGREC often correspond to multiple related but distinct instances, requiring a more structured representation space, an advantage naturally offered by hyperbolic geometry. These findings further confirm the suitability of hyperbolic space for modeling hierarchical semantics. Nevertheless, HEMix consistently outperforms both $\mathrm{Sim}\mathrm{E}$ and $\mathrm{Sim}\mathrm{H}$ across all datasets. Even in cases where $\mathrm{Sim}\mathrm{E}$ and $\mathrm{Sim}\mathrm{H}$ perform similarly (e.g., RefCOCO testA, gRefCOCO testA), HEMix achieves great improvements. This highlights the importance of incorporating both fine-grained and hierarchical information and demonstrates the overall superiority of our proposed HEMix.
\vspace{-3mm}

\paragraph{Ablation Study on Prompt design}
We conduct an ablation study on prompt design and use N-acc as metrics. The prompt $P$ contains $P_G$, $P_C$, $P_E$, $P_Q$ and $I$. Among them, $P_G$, $P_C$, $P_Q$, and $I$ are indispensable; the absence of any one would lead to unparseable VLM output and consequently task failure. Therefore, we present two variants: one with $P_E$ removed, and another, $P^*$, where the textual content is altered while preserving the semantic information of $P$.Removing $P_E$ led to some outputs not following the format well and reduced the model's understanding of the task, which in turn led to performance degradation. The performance of $P^*$ when changing the text content was similar to $P$, indicating the robustness of our method to prompts.
\vspace{-3mm}

\paragraph{Effect of Different Hyperbolic Mapping Designs.}
Prior works~\citep{hypfs,khrulkov2020hyperbolic,hypmeru}
typically adopt the exponential map to project features onto the hyperboloid manifold, preserving the mathematical correctness of hyperbolic embeddings. However, we observe that this mapping introduces steep gradients, which negatively affect training stability and optimization. As shown in Tab.~\ref{ablation:mapall}, replacing the exponential map with a learnable linear layer yields a substantial improvement of 5.24\% on RefCOCO and 1.23\% on gRefCOCO, indicating that a simpler, trainable mapping leads to better empirical performance in practice. To our best knowledge, only these two hyperbolic mapping schemes exist.
\vspace{-3mm}
\paragraph{Cross-Dataset Validation.}
To evaluate the robustness and generalization of our framework, we train the model on the gRefCOCO and test it on standard WREC benchmarks with single-target annotations. As shown in Tab.~\ref{ablation3}, our method consistently outperforms the single-target baseline RefCLIP across all test sets, despite being trained on the same gRefCOCO data. Compared to in-domain training, the cross-dataset setting yields even larger gains (e.g., +5.41\% $\textit{vs.}$ +0.80\% in RefCOCOg), further highlighting the superior transferability of our approach. Notably, our model even surpasses in-domain RefCLIP models on RefCOCO+ and RefCOCOg, demonstrating strong generalization across both task settings and dataset domains.

Furthermore, our more ablation study in Appendix~\ref{Appendix:More Ablation Studies} reveals several consistent trends across datasets. First, sweeping the hybrid weight $\alpha$ for HEMix produces a clear U-shaped curve with a robust sweet spot around $\alpha\!\in\![0.4,0.7]$; for example, RefCOCO testA peaks at $60.01\%$ when $\alpha{=}0.7$, while gRefCOCO testB reaches $36.44\%$ at $\alpha{=}0.9$ (Tab.~\ref{alphaab}). Second, adding an \emph{explicit} hierarchical constraint brings little to no average improvement over the \emph{implicit} structure already captured by HEMix (Tab.~\ref{hierarchicalconstrain}). Third, during referent grounding, excluding $v{=}0$ samples avoids degenerate contrastive updates and improves performance, while such cases are handled at inference by the referential decoupling stage, outputting ``0'' (Tab.~\ref{tab:refclip_v0}). 

\vspace{-4mm}
\subsection{Quantitative Analysis}
\vspace{-3mm}

\begin{figure}
    \centering
    \vspace{-4mm}
    \includegraphics[width=\linewidth]{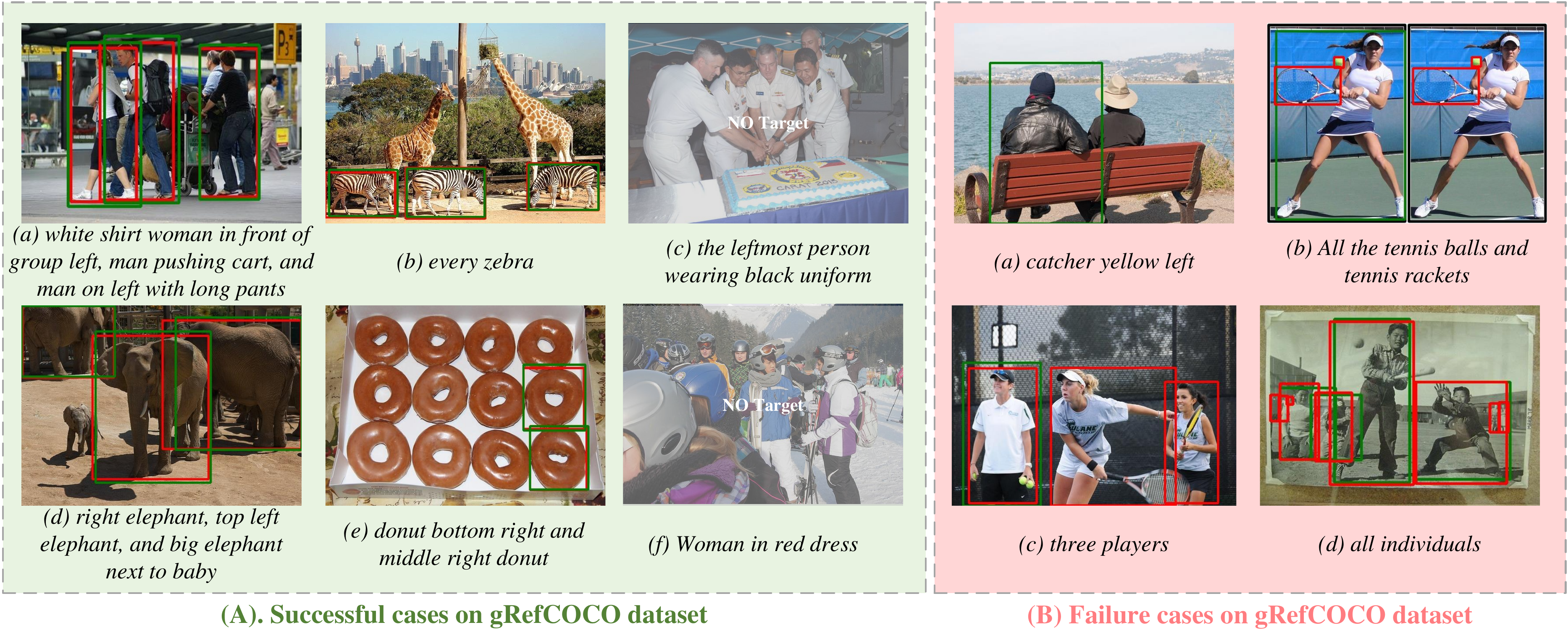}
    \caption{Successful cases(green background) and failure cases(red background). The ground truth is denoted by red bounding boxes, whereas green bounding boxes denote the predictions.}
    \label{Quantitative}
    \vspace{-4mm}
\end{figure}
To further understand the behavior of our model in complex referring scenarios, we visualize representative success and failure cases from the gRefCOCO dataset in Fig.~\ref{Quantitative}. These examples highlight the strengths and current limitations of our framework.
\vspace{-3mm}

\paragraph{Success Cases.} As shown in Fig.~\ref{Quantitative}(A), our method demonstrates strong grounding ability across both multi-target and no-target expressions. For instance, in case (a), the expression includes multiple entities with spatial and appearance constraints, and our model is able to localize each person correctly despite heavy occlusion and crowd density. In (e), although the objects (donuts) are visually similar, our model grounds the correct ones by leveraging position cues. Moreover, our approach can successfully identify no-target cases (c) and (f). In (c), the phrase \textit{``the leftmost person wearing black uniform''} does not correspond to any entity in the image, and our model makes a correct no-target prediction. In (f), while the scene contains many people, none of them match the detailed attributes described in expression (\textit{``woman in red dress''}), and the model again avoids false positives.
\vspace{-3mm}

\paragraph{Failure Cases.} Fig.~\ref{Quantitative}(B) presents typical failure cases. In (a), the expression implies no valid target, yet the model incorrectly grounds an entity, reflecting the limited semantic understanding of the VLM. In (b), the number of targets is predicted correctly, but all decomposed expressions collapse to the same label (``people''), caused by VLM hallucinations. This occurs despite correctly detecting four objects (two balls and two rackets); class imbalance biases the decoder toward the frequent category ``people,'' consistent with prior findings on distribution-induced hallucination~\citep{zhang2024knowledge,mckenna-etal-2023-sources,rohrbach-etal-2018-object,liang2025rbd}. In (c), although the decoupled phrases are semantically valid and distinct, the grounding stage fails to segment precise regions, leading to overlaps and missed detections. In (d), the model misses fine-scale visual cues, such as small hands or blurred individuals, demonstrating difficulty in handling subtle details.

%% file: iclr2026/sections/6_conclusion.tex
\vspace{-3mm}
\section{Conclusion}
\vspace{-3mm}

In this paper, we delve into the critical limitation of Weakly-Supervised Referring Expression Comprehension (WREC) task: the inability of existing methods to handle expressions corresponding to a variable number of targets. To address this, we introduce the Weakly-Supervised Generalized Referring Expression Comprehension (WGREC) task, a more realistic and challenging setting where an expression may refer to multiple, single, or no objects. We then propose LIHE, the first weakly-supervised framework designed for this generalized task. LIHE operates via a two-stage process: (1) Referential Decoupling, where a vision-language model (VLM) infers the number of potential referents and parses the expression into target-relevant sub-phrases, followed by (2) Referent Grounding, where the model enhanced by our novel hybrid similarity localizes each sub-phrase. This design leverages the semantic understanding of VLMs to resolve ambiguity and synergistically combines Euclidean and hyperbolic geometries to preserve hierarchical representations. Despite these strengths, experiments demonstrate that LIHE achieves state-of-the-art performance on both WGREC and WREC benchmarks, validating its robustness and generalization. One key limitation of LIHE is its reliance on VLMs, which limits inference speed ($\approx$2 FPS). Thus, LIHE is more suitable as a teacher model for generating pseudo-labels~\citep{diffusaliency} to supervise smaller, faster student models. In future work, we will explore student model design and lightweight VLM adaptation, as well as extend our hybrid similarity design to richer geometric formulations.

\section*{Reproducibility Statement}
All theoretical results are established under explicit conditions. Experimental details are given in Appendix~\ref{Appendix:Implementation Details}. The code and weights are available at \url{https://anonymous.4open.science/r/LIHE}.
\section*{Ethics Statement}
This work only uses publicly available datasets and does not involve human subjects or sensitive information. We identify no specific ethical concerns.

%% file: iclr2026/sections/appendix.tex
\cleardoublepage
\appendix
\section*{Appendix}
{\hypersetup{linkcolor=black}
\startcontents[sections]
\printcontents[sections]{l}{1}{\setcounter{tocdepth}{2}}
}
\clearpage
\section{Use of large language models}
We use large language models to aid or polish writing. 

\section{Related Work}
\paragraph{Generalized Referring Expression Comprehension}
GREC~\citep{he2023grec} extends traditional REC tasks by permitting linguistic expressions to refer simultaneously to multiple target objects. Baseline methods such as MCN \citep{luo2020multi}, VLT \citep{ding2021vision}, MDETR \citep{kamath2021mdetr}, and UNINEXT \citep{yan2023universal} have been adopted to evaluate performance under this more complex and realistic scenario. Recent advancements include RECANTFormer \citep{hemanthage2024recantformer}, which employs a recursive transformer decoder with adaptive prediction heads to dynamically predict multiple targets, and HieA2G \citep{wang2025hierarchical}, which leverages a hierarchical alignment mechanism to enhance interactions between linguistic phrases and visual objects, thereby capturing fine-grained semantic correlations. However, these methods uniformly rely on fully supervised bounding box annotations. Weakly supervised approaches for GREC remain unexplored, motivating the development of our proposed weakly supervised framework.
\paragraph{Weakly Supervised REC}
Weakly supervised approaches~\citep{gupta2020contrastive,liu2019adaptive,liu2019kpnet,
liu2021rir,sun2021discriminative,wang2021ckd,zhang2020counterfactual} have recently shown promising results in Referring Expression Comprehension (REC), significantly reducing the dependence on expensive bounding box annotations. Unlike fully supervised methods~\citep{deng2021transvg,ho2022yoro,hong2019ltrs,huang2021look,
kamath2021mdetr,liao2020cmcf,liu2019nmtree,liu2019erase,crossdiff,
luo2020multi,sun2021iterative,yang2019fast,yang2020recursive,
zhang2018vc,zhao2022word2pix,zhou2021ginet,zhou2021trar,
zhu2022seqtr}, weakly supervised REC techniques utilize coarser supervision signals, such as image-level or image-text pair annotations, enabling scalable and cost-effective training. Notable one-stage methods~\citep{jin2023refclip,luo2024apl,zhao2018wa}
include RefCLIP \citep{jin2023refclip}, which redefines REC as an anchor-text matching task using anchor-based contrastive learning to align visual and textual features without bounding boxes, and APL \citep{luo2024apl}, which enriches anchor features through position, color, and category prompts, coupled with auxiliary losses to enhance vision-language alignment. Despite their effectiveness in traditional REC tasks, these methods inherently assume a single-target scenario, limiting their generalizability to multi-target WGREC settings. This highlights the need to develop specialized, weakly supervised approaches tailored specifically to GREC.
\paragraph{Hyperbolic Representation learning }
Hyperbolic representation learning~\citep{hyptxt1s,hyptxt2s,hyplorentz,hyppoincare,hier,hypvit,
hypseg1,hypseg2,hypfs,khrulkov2020hyperbolic,hypmeru} leverages hyperbolic geometry's exponential volume growth and negative curvature, making it particularly effective for modeling hierarchical and relational structures. Early seminal works such as Poincaré Embeddings \citep{nickel2017poincare} and Hyperbolic Neural Networks \citep{ganea2018hyperbolic} established foundational techniques for embedding structured data into hyperbolic spaces. Recent developments have adapted hyperbolic geometry to vision and cross-modal tasks. For instance, hyperbolic embedding methods introduced by Kwon et al.~\citep{kwon2024improving} and Kong et al.~\citep{kong2024hyperbolic} have improved visual recognition and open-world detection by preserving hierarchical semantics. For vision-language alignment, Ge et al. \citep{ge2023hyperbolic} and Ramasinghe et al.~\citep{ramasinghe2024accept} used hyperbolic contrastive learning to enhance semantic coherence. Unlike previous methods that use only hyperbolic distance or angle, we fuse Euclidean and hyperbolic similarities into a hybrid metric, which boosts WREC performance and showcases a new way to apply hyperbolic geometry in representation learning.

\section{Implementation Details}
\label{Appendix:Implementation Details}
In the referential decoupling stage, we adopt a pre-trained VLM~\citep{qwen2.5-VL} to understand the visual entity and generate the decomposed referring expression. In the referent grounding stage, following prior work, we resize every image to $416 \times 416$. The maximum token length of the referring expression is fixed to 15 for all datasets. For anchor extraction, we adopt the YOLOv3~\citep{redmon2018yolov3} model pre‑trained on MSCOCO~\citep{lin2014microsoft}, where images from the validation and test splits of those datasets are removed to avoid leakage. Detector weights are frozen during all stages of training. The language encoder produces 512‑dimensional sentence embeddings.  
Visual anchor features are first fused across scales and then linearly projected to the same 512‑dimensional joint space.  
In anchor‑based contrastive learning, we adopt a 512‑dimensional projection head and sample two negative anchors per image by default.  
All WREC tasks are trained on NVIDIA GPU A100 40G and all WGREC tasks on A6000 48G. All models are optimized with AdamW using a constant learning rate of $1e-4$. We train for 25 epochs with a batch size of 64. 

\section{Hyperbolic Space Properties}
\label{Appendix:Hyperbolic Space}
Hyperbolic spaces are Riemannian manifolds characterized by negative curvature, and they differ fundamentally depending on the curvature value. As the curvature approaches zero, the hyperbolic space gradually transitions into a Euclidean space. When the curvature is negative, the space exhibits hyperbolic geometry, where parallel lines can diverge, and the volume grows exponentially with distance from a point. These spaces are commonly represented using various models, including the Poincaré ball, the Lorentz model, and the Klein model, each providing unique advantages for mathematical formulations.
\paragraph{Lorentz Model} The Lorentz Model is also known as the hyperboloid model or the Minkowski model. In the Lorentz model, a hyperbolic \( n \)-dimensional manifold is commonly realized as a sub-manifold of \( \mathbb{R}^{n+1} \), corresponding to the upper sheet of a two-sheeted hyperboloid. Each point \( \mathbf{x} \in \mathbb{R}^{n+1} \) of the Lorentz model, can be represented as \( [\mathbf{x}_{\text{time}}, x_{\text{space}}] \), where \( x_{\text{time}} \in \mathbb{R} \) denotes the temporal component and \( \mathbf{x}_{\text{space}} \in \mathbb{R}^n \) denotes the spatial components. The \( n \)-dimensional Hyperbolic space \( \mathbb{H}^{n} \) with curvature $\kappa$ represented by a \({n+1}\)-dimensional Lorentz Model as follows:
\begin{equation}
\mathbb{H}^n_\kappa \;=\;\Bigl\{\mathbf{x}\in\mathbb{R}^{n+1}\ \bigl|\ 
\langle\mathbf{x},\mathbf{x}\rangle_\mathbb{H} \;=\; -\kappa^{-1},\; x_0>0\Bigr\}
\label{eq:hyperb}
\end{equation}
where the Lorentzian inner product is defined as,
\begin{equation}
\langle\mathbf{x},\mathbf{y}\rangle_\mathbb{H} \;=\; -x_0y_0 \;+\; \sum_{i=1}^{n} x_i y_i.
\label{eq:lor_inner}
\end{equation}
Here, the 0-th dimension of the vector \(\mathbf{x}\), \(x_0\) is treated as the time component \(x_{time}\) and the rest dimension of the vector \(\mathbf{x}\), \(\mathbf{x}_{1:n}\) is the space component \(\mathbf{x}_{space}\). The \(x_{time}\) can be calculated from \(\mathbf{x}_{space}\)
as follows:
\begin{equation}
x_{time} = x_0 =\sqrt{\|\mathbf{x}_{\mathrm{space}}\|^{2} + \kappa^{-1}}\,.
\label{xtime}
\end{equation}
where the $\| \dotsc \|$ is the Euclidean norm.

\paragraph{Distance}
In the Lorentz model, the geodesic distance between two points 
$\mathbf{x},\mathbf{y}\in\mathbb{H}^{n}_{\kappa}$ ($\kappa>0$) can be expressed
solely in terms of their Lorentzian inner product. The
distance function $d_{\kappa}:\mathbb{H}^{n}_{\kappa}\times
\mathbb{H}^{n}_{\kappa}\to\mathbb{R}_{\ge 0}$ is
\begin{equation}
d_{\kappa}(\mathbf{x},\mathbf{y}) \;=\;
\frac{1}{\sqrt{\kappa}}\,
\operatorname{arccosh}\!\Bigl(
  -\,\kappa\,
  \langle\mathbf{x},\mathbf{y}\rangle_{\mathbb{H}}
\Bigr),
\label{eq:lorentz_distance}
\end{equation}
where $\operatorname{arccosh}$ is the inverse hyperbolic cosine and $\langle\cdot,\cdot\rangle_{\mathbb{H}}$ is lorentz inner product in~\eqref{eq:lor_inner}.  The quantity
inside the $\operatorname{arccosh}$ is always $\ge 1$ for points on the
hyperboloid, guaranteeing that the distance is real‐valued.  This formula
highlights a key feature of hyperbolic geometry: the distance grows
logarithmically with the Lorentzian inner product, reflecting the exponential
volume growth characteristic of negatively curved spaces.

\paragraph{Exponential Map}
Given a base point $\mathbf{p}\in\mathbb{H}^{n}_{\kappa}$ and a tangent
vector $\mathbf{v}\in T_{\mathbf{p}}\mathbb{H}^{n}_{\kappa}$ (Euclidean space), the
\emph{exponential map}
$\operatorname{Exp}^{\kappa}_{\mathbf{p}}:T_{\mathbf{p}}\mathbb{H}^{n}_{\kappa}\to
\mathbb{H}^{n}_{\kappa}$ moves $\mathbf{p}$ along the unique geodesic in the
direction of~$\mathbf{v}$.  the map is
\begin{equation}
\operatorname{Exp}^{\kappa}_{\mathbf{p}}(\mathbf{v})
\;=\;
\cosh\!\bigl(\sqrt{\kappa}\,\|\mathbf{v}\|\bigr)\,\mathbf{p}
\;+\;
\frac{\sinh\!\bigl(\sqrt{\kappa}\,\|\mathbf{v}\|\bigr)}
     {\sqrt{\kappa}\,\|\mathbf{v}\|}\,\mathbf{v},
\label{eq:exp_map}
\end{equation}
where $\|\cdots\|$ for the Lorentz norm o, 
For small $\|\mathbf{v}\|$ this reduces to $\mathbf{p}+\mathbf{v}$, mirroring
the Euclidean limit, while for large $\|\mathbf{v}\|$ the hyperbolic
$\cosh/\sinh$ terms dominate, capturing the curvature‐induced stretching of
space.

\paragraph{Logarithm Map}
Conversely, the \emph{logarithm map}
$\operatorname{Log}^{\kappa}_{\mathbf{p}}:\mathbb{H}^{n}_{\kappa}\to
T_{\mathbf{p}}\mathbb{H}^{n}_{\kappa}$ sends a point $\mathbf{q}$ back to the
tangent space at~$\mathbf{p}$, producing the initial velocity vector of the
geodesic from $\mathbf{p}$ to $\mathbf{q}$.  Using the distance
$d_{\kappa}(\mathbf{p},\mathbf{q})$ from~\eqref{eq:lorentz_distance}, we have
\begin{equation}
\operatorname{Log}^{\kappa}_{\mathbf{p}}(\mathbf{q})
\;=\;
\frac{d_{\kappa}(\mathbf{p},\mathbf{q})}
     {\sinh\!\bigl(\sqrt{\kappa}\,d_{\kappa}(\mathbf{p},\mathbf{q})\bigr)}
\left(
  \mathbf{q}
  +\kappa\,\langle\mathbf{p},\mathbf{q}\rangle_{\mathbb{H}}\;\mathbf{p}
\right),
\label{eq:log_map}
\end{equation}
which indeed satisfies
$\operatorname{Exp}^{\kappa}_{\mathbf{p}}\!\bigl(\operatorname{Log}^{\kappa}_{\mathbf{p}}(\mathbf{q})\bigr)
=\mathbf{q}$.  The factor in front rescales the component of $\mathbf{q}$
orthogonal to $\mathbf{p}$ so that its norm equals the hyperbolic distance,
giving a first‐order approximation to motion on the manifold that is exact
along geodesics.

From these definitions, we highlight two key insights:\textbf{(1)} \textbf{Any vector that satisfies the hyperboloid constraint (Eq.~\eqref{eq:hyperb}) is a valid point on the Lorentz manifold and inherits the geometric properties of hyperbolic space.}\textbf{(2)} \textbf{A higher Lorentzian inner product indicates greater semantic similarity between two points in hyperbolic space.} 

\begin{proposition}[Hyperboloid Membership]
Let $\kappa>0$.  
A vector $\mathbf{x}\in\mathbb{R}^{\,n+1}$ with $x_0>0$ satisfies the hyperboloid
constraint
\[
\langle\mathbf{x},\mathbf{x}\rangle_{\mathbb{H}}=-\kappa^{-1}
\quad\Longleftrightarrow\quad
\mathbf{x}\in\mathbb{H}^{n}_{\kappa}.
\]
Consequently every such $\mathbf{x}$ is a valid point of the Lorentz (hyperboloid)
model of the hyperbolic space of constant sectional curvature $-\kappa$,
and inherits all of its geometric properties.
\end{proposition}

\begin{proof}
\emph{($\Rightarrow$)}  
By definition
\[
\mathbb{H}^{n}_{\kappa}\;=\;
\Bigl\{\mathbf{z}\in\mathbb{R}^{\,n+1}\ \bigl|\ 
\langle\mathbf{z},\mathbf{z}\rangle_{\mathbb{H}}=-\kappa^{-1},\ z_0>0
\Bigr\},
\]
so any $\mathbf{x}$ satisfying the stated constraint (with $x_0>0$)
belongs to $\mathbb{H}^{n}_{\kappa}$.

\smallskip
\noindent\emph{($\Leftarrow$)}  
Conversely, if $\mathbf{x}\in\mathbb{H}^{n}_{\kappa}$,
then by the same defining condition we have
$\langle\mathbf{x},\mathbf{x}\rangle_{\mathbb{H}}=-\kappa^{-1}$ and $x_0>0$.
Hence the two sets coincide, establishing the equivalence.
\end{proof}

\begin{proposition}[Monotonicity of the Lorentzian Inner Product]
For any two points
$\mathbf{x},\mathbf{y}\in\mathbb{H}^{n}_{\kappa}$ with $\kappa>0$, the geodesic
distance~\eqref{eq:lorentz_distance}
\[
d_{\kappa}(\mathbf{x},\mathbf{y})
=\frac{1}{\sqrt{\kappa}}\,
\operatorname{arccosh}\!\bigl(-\kappa\,\langle\mathbf{x},\mathbf{y}\rangle_{\mathbb{H}}\bigr)
\]
is a strictly \emph{decreasing} function of the Lorentzian inner product
$\langle\mathbf{x},\mathbf{y}\rangle_{\mathbb{H}}$.
Equivalently, a \textbf{larger inner product indicates
smaller hyperbolic distance} and therefore higher semantic similarity.
\end{proposition}

\begin{proof}
Define
\[
u \;:=\;
-\kappa\,\langle\mathbf{x},\mathbf{y}\rangle_{\mathbb{H}}
\quad\bigl(\;u\ge 1\;\bigr)
\]
\[
f(u)\;:=\;\frac{1}{\sqrt{\kappa}}\operatorname{arccosh}(u)
=\;d_{\kappa}(\mathbf{x},\mathbf{y}).
\]
Because $\operatorname{arccosh}$ is strictly increasing on $[1,\infty)$ and
\[
\frac{du}{d\langle\mathbf{x},\mathbf{y}\rangle_{\mathbb{H}}}
=-\kappa<0,
\]
the chain rule gives
\[
\frac{d\,d_{\kappa}(\mathbf{x},\mathbf{y})}
     {d\langle\mathbf{x},\mathbf{y}\rangle_{\mathbb{H}}}
=\frac{df}{du}\,\frac{du}{d\langle\mathbf{x},\mathbf{y}\rangle_{\mathbb{H}}}
<0.
\]
Hence $d_{\kappa}$ decreases strictly as
$\langle\mathbf{x},\mathbf{y}\rangle_{\mathbb{H}}$ increases.  
Since hyperbolic distance quantifies dissimilarity, the inverse relationship
asserts that a larger (i.e.,\ less negative) Lorentzian inner product encodes
greater semantic similarity between the points. 
\end{proof}
Hence, we adopt the Lorentz inner product as a similarity measure. 

\begin{proposition}[Bias--Variance Reduction]
\label{prop:var-reduce2}
Let $\sigma^{2}_{\mathrm{E}} \!=\! \operatorname{Var}[\varepsilon_{\mathrm{E}}]$,
      $\sigma^{2}_{\mathrm{H}} \!=\! \operatorname{Var}[\varepsilon_{\mathrm{H}}]$,
and $\rho \!=\! \operatorname{Corr}[\varepsilon_{\mathrm{E}},\varepsilon_{\mathrm{H}}]$.
Denote the Euclidean and Hyperbolic biases by
$b_{\mathrm{E}}\!=\!\mathbb E[\text{Sim}_{\mathrm{E}}\!-\!\text{Sim}^{\star}]$  
and $b_{\mathrm{H}}\!=\!\mathbb E[\text{Sim}_{\mathrm{H}}\!-\!\text{Sim}^{\star}]$.
For any mixing weight $\alpha\!\in\![0,1]$ define the hybrid estimator
\(
\text{HEMix} \;=\;
(1-\alpha)\text{Sim}_{\mathrm{E}} + \alpha\text{Sim}_{\mathrm{H}}.
\)

Then the mean--squared error (MSE) of \text{HEMix}
\begin{align}
\operatorname{MSE}(\text{HEMix})
&=\mathbb E\!\bigl[(\text{HEMix}-\text{Sim}^{\star})^{2}\bigr] \notag\\
&=((1-\alpha)b_{\mathrm{E}}+\alpha b_{\mathrm{H}})^{2}
 +(1-\alpha)^{2}\sigma_{\mathrm{E}}^{2}
 +\alpha^{2}\sigma_{\mathrm{H}}^{2}
 +2\alpha(1-\alpha)\rho\sigma_{\mathrm{E}}\sigma_{\mathrm{H}}
\label{eq:mse-mix}
\end{align}
is a strictly convex quadratic in $\alpha$.  
Its unique minimizer is
\begin{equation}
\alpha^{\star}=
\frac{\,
      (\sigma_{\mathrm{E}}^{2}+\rho\sigma_{\mathrm{E}}\sigma_{\mathrm{H}})
      +b_{\mathrm{E}}\!\left(b_{\mathrm{E}}-b_{\mathrm{H}}\right)
     \,}
     {\,
      \sigma_{\mathrm{E}}^{2}+\sigma_{\mathrm{H}}^{2}
      -2\rho\sigma_{\mathrm{E}}\sigma_{\mathrm{H}}
      +(b_{\mathrm{E}}-b_{\mathrm{H}})^{2}
     \,},
\end{equation}
which always lies in $(0,1)$ whenever $\rho<1$ or $b_{\mathrm{E}}\neq b_{\mathrm{H}}$.
Moreover,
\begin{equation}
\operatorname{MSE}\bigl(\text{HEMix};\alpha^{\star}\bigr)
\;<\;
\min\!
\bigl\{
      \operatorname{MSE}(\text{Sim}_{\mathrm{E}}),\;
      \operatorname{MSE}(\text{Sim}_{\mathrm{H}})
\bigr\}.
\end{equation}
\end{proposition}

\begin{proof}
Let
\(
f(\alpha)\!=\!\operatorname{MSE}(\text{HEMix})
\)
in \eqref{eq:mse-mix}.
Write it as
$
f(\alpha)=A\alpha^{2}+2B\alpha+C
$
with
\[
\begin{aligned}
A &= (b_{\mathrm{H}}-b_{\mathrm{E}})^{2}
     +\sigma_{\mathrm{H}}^{2}+\sigma_{\mathrm{E}}^{2}
     -2\rho\sigma_{\mathrm{E}}\sigma_{\mathrm{H}} \;>\; 0,\\
B &= -\bigl[(b_{\mathrm{H}}-b_{\mathrm{E}})b_{\mathrm{E}}
             +\sigma_{\mathrm{H}}^{2}
             -\rho\sigma_{\mathrm{E}}\sigma_{\mathrm{H}}\bigr],\\
C &= b_{\mathrm{E}}^{2}+\sigma_{\mathrm{E}}^{2}=f(0).
\end{aligned}
\]
Since $A>0$, $f$ is strictly convex; the stationary point
$\alpha^{\star}\!=\!-B/A$ is the global minimum, yielding
$
f(\alpha^{\star}) = C - B^{2}/A .
$
Because $B^{2}/A>0$, we have $f(\alpha^{\star})<f(0)=\operatorname{MSE}(\text{Sim}_{\mathrm{E}})$.
Convexity further implies
$
f(\alpha^{\star})<\max\{f(0),f(1)\}.
$
Whenever $f(1)\!\neq\!f(0)$ (i.e.,\ the two single-space estimators do not have identical MSE) this gives
$
f(\alpha^{\star})<\min\{f(0),f(1)\},
$
which is exactly the desired inequality.  
\end{proof}

\section{Detailed Limitations of RefCLIP}
\label{Appendix:LimitationsofRefCLIP}
\begin{figure*}
    \centering
    \setlength{\abovecaptionskip}{6pt}   
    \setlength{\belowcaptionskip}{1pt}  
    \includegraphics[width=1.0\linewidth]{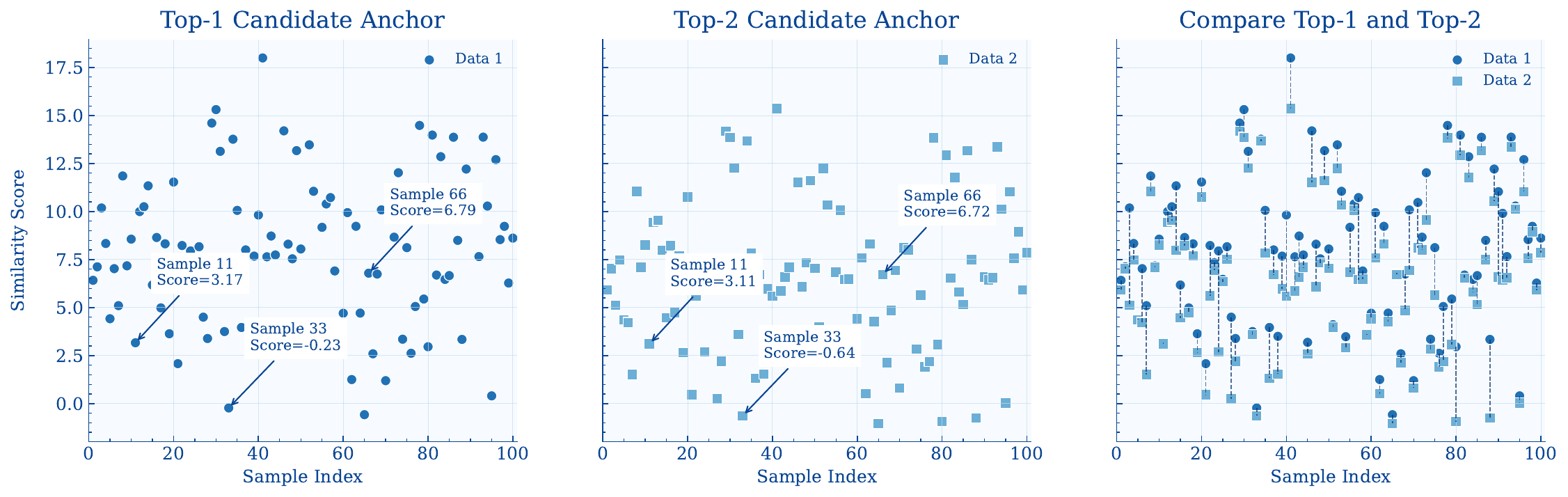}
    \caption{We randomly select 100 samples from the dataset and visualize their similarity scores. From left to right: (1) the similarity scores of top-scoring anchors (dark blue dots), (2) the similarity scores of second-best anchors (light blue squares), and (3) an overlaid view combining both. The dashed vertical segments connect the top and second-best scores for each sample, illustrating that second-best anchors in many cases have higher similarity than top anchors in other samples.}
    \label{Similarity_Score}
\end{figure*}
As the main text says, weakly supervised methods such as RefCLIP effectively simplify the REC task by reducing it to an anchor-text matching problem. Specifically, the anchor selection mechanism in methods like RefCLIP can be expressed as:
\begin{equation}
a^* = \arg\max_{a \in \mathcal{A}} \phi(T, I, a),
\label{RECTask2}
\end{equation}
where \( \phi(T, I, a) \) represents the similarity between the text expression \( T \), the image \( I \), and anchor \( a \) from the anchor set \( \mathcal{A} \).
However, this max-selection strategy implicitly assumes that the number of referred objects is known in advance, making it unsuitable for Generalized Referring Expression Comprehension (GREC), where the number of targets is unknown. 

A natural alternative is to apply a threshold to filter anchors based on their similarity scores. Yet, our experimental analysis reveals the limitations of this approach. As shown in Fig.~\ref{Similarity_Score}, the second-best anchor in some samples exhibits higher similarity than the top-scoring anchors in others, rendering a universal threshold ineffective for consistent selection. This highlights the inadequacy of threshold-based selection under WGREC conditions.

\section{Evaluation Metrics}

\label{appendx:metrics}

\paragraph{\textbf{Precision@(F$_1$=1, IoU$\ge$0.5)}}
For each sample, let \(\mathcal{G}\) and \(\mathcal{P}\) be the ground-truth and predicted
bounding-box sets.  A prediction is \textbf{matched} to a ground-truth box if their
intersection-over-union (IoU) is at least~0.5.  
If several predictions match the \emph{same} ground-truth box, only the one with
the highest IoU is kept as a true positive (\textsc{tp}); the rest are false positives
(\textsc{fp}).  
Unmatched ground-truth boxes are false negatives (\textsc{fn}).  
The sample-level F\(_1\) score is
\[
F_1 \;=\; \frac{2\;{TP}}{2\;TP + FP + FN}.
\]
For \emph{no-target} samples (\(|\mathcal{G}|=0\)), we set \(F_1=1\) if
\(|\mathcal{P}|=0\) and \(F_1=0\) otherwise.

\noindent
\textbf{Precision@(F$_1$=1, IoU$\ge$0.5)} is the proportion of samples whose
\(F_1\) score equals~1:
\[
\operatorname{Precision@}(F_1{=}1,\;IoU{\ge}0.5)
\;=\;
\frac{1}{N}\sum_{i=1}^{N}\mathds{1}\!\left[F_1^i=1\right].
\]
It reports the percentage of images perfectly predicted (no missed or
spurious detections) under the 0.5 IoU criterion.

\paragraph{N-acc.} quantifies the model’s ability to correctly identify
\emph{no-target} samples—images that contain no ground-truth objects.
For such a sample,
\begin{itemize}
  \item \textbf{true positive (TP):} the model predicts \emph{no} bounding boxes;
  \item \textbf{false negative (FN):} the model predicts at least one bounding box.
\end{itemize}
The metric is
\[
  \text{N-acc.}\;=\;
  \frac{\text{TP}}{\text{TP}+\text{FN}},
\]
the proportion of no-target samples for which the model outputs no detections.

\section{Detailed Prompt Design}
\label{Appendix:DetailedPromptDesign}
$\mathbf{P_G}$: 
Task Explanation: You need to process an image and a referring expression. The image may contain zero, one, or multiple target objects corresponding to the referring expression. Analyze the image to determine whether the target exists. If the target does not exist or the referring expression is empty, output a single number "0". If the target exists, output the number of targets and generate a unique referring expression for each target. The referring expressions must describe distinct targets unambiguously using attributes like color, position, size, etc. 

$\mathbf{P_C}$: 
You should provide a number indicating how many targets exist in the image, and then describe each target with a short, distinct phrase. Prefix each phrase with its ordinal number. The number of targets is extremely important — please check carefully. The phrases must be accurate and distinct. 

$\mathbf{P_E}$: 
For example, if the referring expression is ``3 people", you should output: 
``3\textbackslash n1. person ...\textbackslash n2. person ...\textbackslash n3. person ..."
The word ``and" is generally used between two target items. 

$\mathbf{P_Q}$: 
The referring expression is: \{referring expression\}

The specific usage process of prompts are shown in Tab.~\ref{tab:supp_beach1}, Tab.~\ref{tab:supp_beach2} and Tab.~\ref{tab:supp_beach3}:

\begin{table*}[!t]
\label{Tab:prompt1}
\centering
\captionsetup{skip=4pt}
\begin{tcolorbox}[enhanced,
  width=0.98\textwidth,   
  colback=white, colframe=black!20,
  boxrule=0.4pt, arc=2pt,
  left=6pt,right=6pt,top=6pt,bottom=6pt]
  \begin{tabularx}{\linewidth}{>{\raggedright\arraybackslash}X m{4.2cm}}
  \textbf{System}: Task Explanation: You need to process an image and a referring expression. The image may contain zero, one, or multiple target objects corresponding to the referring expression. Analyze the image to determine whether the target exists. If the target does not exist or the referring expression is empty, output a single number "0". If the target exists, output the number of targets and generate a unique referring expression for each target. The referring expressions must describe distinct targets unambiguously using attributes like color, position, size, etc. \\
    \textbf{User}: The referring expression is: \{the right boy in black shirt is playing skateboard\}. You should provide a number indicating how many targets exist in the image, and then describe each target with a short, distinct phrase. Prefix each phrase with its ordinal number. The number of targets is extremely important — please check carefully. The phrases must be accurate and distinct. For example, if the referring expression is ``3 people", you should output: ``3\textbackslash n1. person ...\textbackslash n2. person ...\textbackslash n3. person ..."The word ``and" is generally used between two target items. 
\\[0.4em]
    \textbf{Referential Decoupling}: ``0"\\
    \textbf{\textcolor{red}{No Referent Grounding and directly Return No-target}} \\
    \centering
    \includegraphics[width=0.35\linewidth]{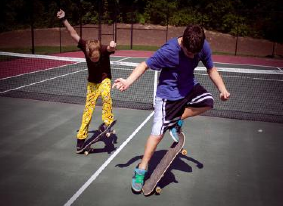}
  \end{tabularx}
\end{tcolorbox}

\caption{No-target case.}
\label{tab:supp_beach1}
\end{table*}

\begin{table*}[!t]
\label{Tab:prompt2}
\centering
\captionsetup{skip=4pt}
\begin{tcolorbox}[enhanced,
  width=0.98\textwidth,   
  colback=white, colframe=black!20,
  boxrule=0.4pt, arc=2pt,
  left=6pt,right=6pt,top=6pt,bottom=6pt]
  \begin{tabularx}{\linewidth}{>{\raggedright\arraybackslash}X m{4.2cm}}
  \textbf{System}: Task Explanation: You need to process an image and a referring expression. The image may contain zero, one, or multiple target objects corresponding to the referring expression. Analyze the image to determine whether the target exists. If the target does not exist or the referring expression is empty, output a single number "0". If the target exists, output the number of targets and generate a unique referring expression for each target. The referring expressions must describe distinct targets unambiguously using attributes like color, position, size, etc. \\
    \textbf{User}: The referring expression is: \{a guy in green and a rightmost guy\}. You should provide a number indicating how many targets exist in the image, and then describe each target with a short, distinct phrase. Prefix each phrase with its ordinal number. The number of targets is extremely important — please check carefully. The phrases must be accurate and distinct. For example, if the referring expression is ``3 people", you should output: ``3\textbackslash n1. person ...\textbackslash n2. person ...\textbackslash n3. person ..."The word ``and" is generally used between two target items. 
\\[0.4em]
    \textbf{Referential Decoupling}: ``2\textbackslash n1.a guy in green\textbackslash n2.a rightmost guy"\\
    \textbf{Referent Grounding:} \\
    \centering
    \includegraphics[width=0.35\linewidth]{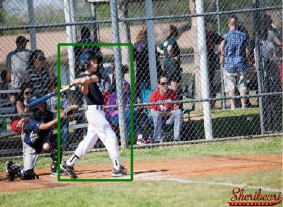}
  \end{tabularx}
\end{tcolorbox}

\caption{One-target case.}
\label{tab:supp_beach2}
\end{table*}

\begin{table*}[!t]
\label{Tab:prompt3}
\centering
\captionsetup{skip=4pt}
\begin{tcolorbox}[enhanced,
  width=0.98\textwidth,   
  colback=white, colframe=black!20,
  boxrule=0.4pt, arc=2pt,
  left=6pt,right=6pt,top=6pt,bottom=6pt]
  \begin{tabularx}{\linewidth}{>{\raggedright\arraybackslash}X m{4.2cm}}
  \textbf{System}: Task Explanation: You need to process an image and a referring expression. The image may contain zero, one, or multiple target objects corresponding to the referring expression. Analyze the image to determine whether the target exists. If the target does not exist or the referring expression is empty, output a single number "0". If the target exists, output the number of targets and generate a unique referring expression for each target. The referring expressions must describe distinct targets unambiguously using attributes like color, position, size, etc. \\
    \textbf{User}: The referring expression is: \{three glasses\}. You should provide a number indicating how many targets exist in the image, and then describe each target with a short, distinct phrase. Prefix each phrase with its ordinal number. The number of targets is extremely important — please check carefully. The phrases must be accurate and distinct. For example, if the referring expression is ``3 people", you should output: ``3\textbackslash n1. person ...\textbackslash n2. person ...\textbackslash n3. person ..."The word ``and" is generally used between two target items. 
\\[0.4em]
    \textbf{Referential Decoupling}: ``2\textbackslash n1.first glass is on the left\textbackslash n2.second glass is in the
middle\textbackslash n3. third glass is on the right side"\\
    \textbf{Referent Grounding:} \\
    \centering
    \includegraphics[width=0.35\linewidth]{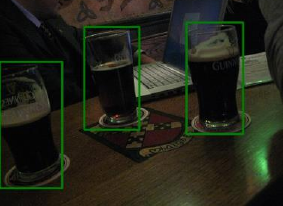}
  \end{tabularx}
\end{tcolorbox}

\caption{Multi-target case.}
\label{tab:supp_beach3}
\end{table*}

\clearpage

\section{More Quantitative Results}
\label{MoreQuantitativeStudies}
\subsection{Unsupervised Generalized Referring Expression Comprehension}
\label{Unsupervised Schema}

To demonstrate the robustness of our framework and explore its potential in a zero-annotation scenario, we extend LIHE to a fully unsupervised setting. In this setting, the model is trained without using any manually annotated data---neither annotated bounding boxes nor corresponding language queries. The only input is the image itself.

To achieve this, we first leverage the Vision-Language Model (VLM) to automatically generate a set of candidate "pseudo referring expressions" for each image. Subsequently, these machine-generated texts are used as the training data for the second stage of our model. As shown in the ``Generated'' row of Tab.~\ref{nosuper}, the experimental results demonstrate that despite relying entirely on machine-generated text, this unsupervised approach still achieves 32.29\%/25.85\%/27.91\% performance on the gRefCOCO val/testA/testB splits, respectively. This result is only about 7-9\% lower than its weakly supervised counterpart trained with authentic, human-annotated text, which strongly demonstrates LIHE's effectiveness in an annotation-free environment. 

\begin{table}[t]
\centering
\footnotesize
\caption{Unsupervised Schema. The bottom row is only using the generated data for training.}
\renewcommand{\arraystretch}{1.2}
\begin{tabular}{l|ccc}
\hline
\multirow{2}{*}{\textbf{Training Set}} & 
\multicolumn{3}{c}{\textbf{gRefCOCO}}  \\
&  val & testA & testB  \\
\hline
 gRefCOCO & 39.61 & 32.19 & 36.44  \\
 Generated & 32.29  &	25.85 &	27.91 \\
\hline
\end{tabular}
\label{nosuper}
\end{table}
\subsection{More Ablation Studies}
\label{Appendix:More Ablation Studies}
\paragraph{Impact of the mixing weight $\alpha$.}
Due to \(\rho\) being unknown, we conduct extensive experiments on the hybrid weight $\alpha$ as shown in Tab.~\ref{alphaab} to find the best weight.
Two clear trends emerge.  
\emph{(i)~U-shaped curve.}  
Extremely small $\alpha$ (closer to a pure Euclidean view) and extremely large $\alpha$ (approaching the hyperbolic-only view) both hurt performance; in every split, the scores first rise, peak in the mid-range, and then drop again.  
\emph{(ii)~Robust sweet-spot.}  
The interval $\alpha\!\approx\!0.4$–$0.7$ consistently delivers the best or near-best numbers across all four benchmarks.  For example, on RefCOCO testA the top accuracy of 60.01\% is achieved at $\alpha\!=\!0.7$, while gRefCOCO testB peaks at 36.44\% with $\alpha\!=\!0.9$, confirming that a balanced mixture captures complementary information from both geometries.  
Overall, the ablation validates the bias-variance analysis: an appropriate hybrid weighting outperforms either single-space similarity.

\begin{table}[t]
\centering
\footnotesize
\caption{Ablation study on \(\alpha\).}
\renewcommand{\arraystretch}{1.3}
\begin{tabular}{l|ccc|ccc|c|ccc}
\hline
\multirow{2}{*}{{$\mathbf{\alpha}$}} &
\multicolumn{3}{c|}{\textbf{RefCOCO}} &
\multicolumn{3}{c|}{\textbf{RefCOCO+}} &
\textbf{RefCOCOg} &
\multicolumn{3}{c}{\textbf{gRefCOCO}} \\[-0.2em]
& val & testA & testB & val & testA & testB & val & val & testA & testB \\
\hline
0.1 & 58.73 & 57.54 & 56.09 & 40.73 & 40.78 & 39.27 & 47.51 & 38.48 & 32.21 & 34.47 \\
0.2 & 61.27 & 60.16 & 59.20 & 41.54 & 41.84 & 39.13 & 47.83 & 38.97 & 32.64 & 34.86 \\
0.3 & 60.88 & 59.08 & 58.33 & 41.14 & 41.18 & 39.09 & 47.22 & 39.07 & 32.22 & 36.27 \\
0.4 & 60.25 & 59.09 & 57.59 & 42.42 & 42.63 & 40.03 & 48.28 & 39.20 & 32.01 & 35.79 \\
0.5 & 60.95 & 59.84 & 58.57 & 41.48 & 42.54 & 39.37 & 48.67 & 39.14 & 32.01 & 35.73 \\
0.6 & 60.05 & 59.06 & 59.02 & 42.04 & 41.58 & 38.17 & 47.55 & 39.25 & 31.71 & 34.97 \\
0.7 & 61.04 & 60.01 & 58.45 & 42.66 & 42.94 & 39.15 & 47.81 & 39.64 & 32.49 & 35.92 \\
0.8 & 60.09 & 58.21 & 58.80 & 42.20 & 43.42 & 39.15 & 45.75 & 39.57 & 32.25 & 35.77 \\
0.9 & 60.43 & 59.50 & 57.59 & 41.52 & 41.90 & 38.84 & 46.49 & 39.61 & 32.19 & 36.44 \\
\hline
\end{tabular}
\label{alphaab}
\end{table}


\paragraph{Explicit $\textit{vs.}$ Implicit Hierarchical Constraints.}
Tab.~\ref{hierarchicalconstrain} investigates whether explicitly adding hierarchical losses helps. In the contrastive loss function (Eq.~\ref{constrastive}), similarity only involves the expression and the anchor vision feature, but we want to know whether it learns the whole hierarchy implicitly. To validate this, we add an extra explicit hierarchical constraint loss as follows:
\begin{align*}
\mathcal{L}_{\text{hier}} = d_{\kappa}(f_{cat}, f_{base\_ref}) + d_{\kappa}(f_{ref}, f_{base\_ref})
\end{align*}
where \(f_{cat}\) denotes the feature of the category text (e.g., `person' as shown in \ref{heratical}), \(f_{base\_ref}\) denotes the feature of the raw referring expression (e.g., `left person') and \(f_{ref}\) denotes the feature of the decomposed referring expression.
As shown in Tab.~\ref{hierarchicalconstrain}, adding the explicit hierarchical constraint basically has the same average performance, implying that through hyperbolic similarity, the model is able to learn Hierarchical Constraints Implicitly.
\paragraph{Training dataset with/without v=0} The Referent Grounding stage is based on the problem definition and optimization of RefClIP which do not support to handle v=0 samples: $a^* = \arg\max_{a \in \mathcal{A}} \phi(T, I, a)$. 
RefCLIP is based on the image-ref pair for contrastive loss, which the process of performing contrastive loss requires the inclusion of a positive sample; when v=0, there is no positive sample available for training the model, which will only damage the model performance. Therefore, in LIHE training stage, we directly remove the v=0 samples from Referent Grounding and in LIHE inference stage, v=0 samples are effectively addressed in the Referential Decoupling stage (Stage 1), where the model outputs ``0" to indicate no detections.

\begin{table}[t]
\centering
\caption{RefCLIP trained with/without $v=0$ sample}
\label{tab:refclip_v0}
\begin{tabular}{lccc}
\toprule
\textbf{RefCLIP trained} & \multicolumn{3}{c}{\textbf{WGREC}} \\
 & val & testA & testB \\
\midrule
With $v=0$    & 16.77 & 15.78 & 20.53 \\
Without $v=0$ & 17.85 & 18.23 & 21.89 \\
\bottomrule
\end{tabular}
\end{table}

\begin{table}[t]
\centering
\caption{Ablation on explicit hierarchical constraints on gRefCOCO.}
\setlength{\tabcolsep}{1.5mm}
\renewcommand{\arraystretch}{1.1}
\begin{tabular}{c|c|c|ccc}
\hline
\multirow{2}{*}{\(Sim_E\)} & \multirow{2}{*}{\(Sim_H\)} & \multirow{2}{*}{\makecell{Hierarchica\\ Constraint}} & \multicolumn{3}{c}{gRefCOCO} \\
 & &  & val & testA & testB \\
\hline
 \checkmark    &               &               & 38.88	& 31.77	 &	34.89   \\
               &  \checkmark   &               & 39.58	 & 31.57  &	 36.88   \\
\rowcolor{gray!15}
\checkmark     &  \checkmark   &               & 39.61  & 32.70  &  35.84   \\
\checkmark     &               & \checkmark    & 39.08  & 32.33  &	35.17 \\
               & \checkmark    & \checkmark    & 39.25  & 32.06  &	 36.03   \\
\checkmark     & \checkmark    & \checkmark    & 39.37	& 32.44	 &  36.62    \\
\hline
\end{tabular}
\label{hierarchicalconstrain}
\end{table}

\vspace{-3mm}
\section{More Visualizations}
In this section, we provide additional qualitative visualizations to further elucidate the behavior of the LIHE framework in various complex scenarios. As shown in Fig.~\ref{visualize0}, Fig.~\ref{visualize1}, Fig.~\ref{visualize2}, Fig.~\ref{visualize3}, Fig.~\ref{visualize4}, we demonstrate how the model successfully performs referential decoupling (RD) on the input referring expressions (RE) to accurately identify multi-target, single-target, or no-target situations. These diverse cases intuitively demonstrate the effectiveness and generalization capability of our proposed method.

\begin{figure*}
    \centering
    \setlength{\abovecaptionskip}{6pt}   
    \setlength{\belowcaptionskip}{1pt}  
    \includegraphics[width=1.0\linewidth]{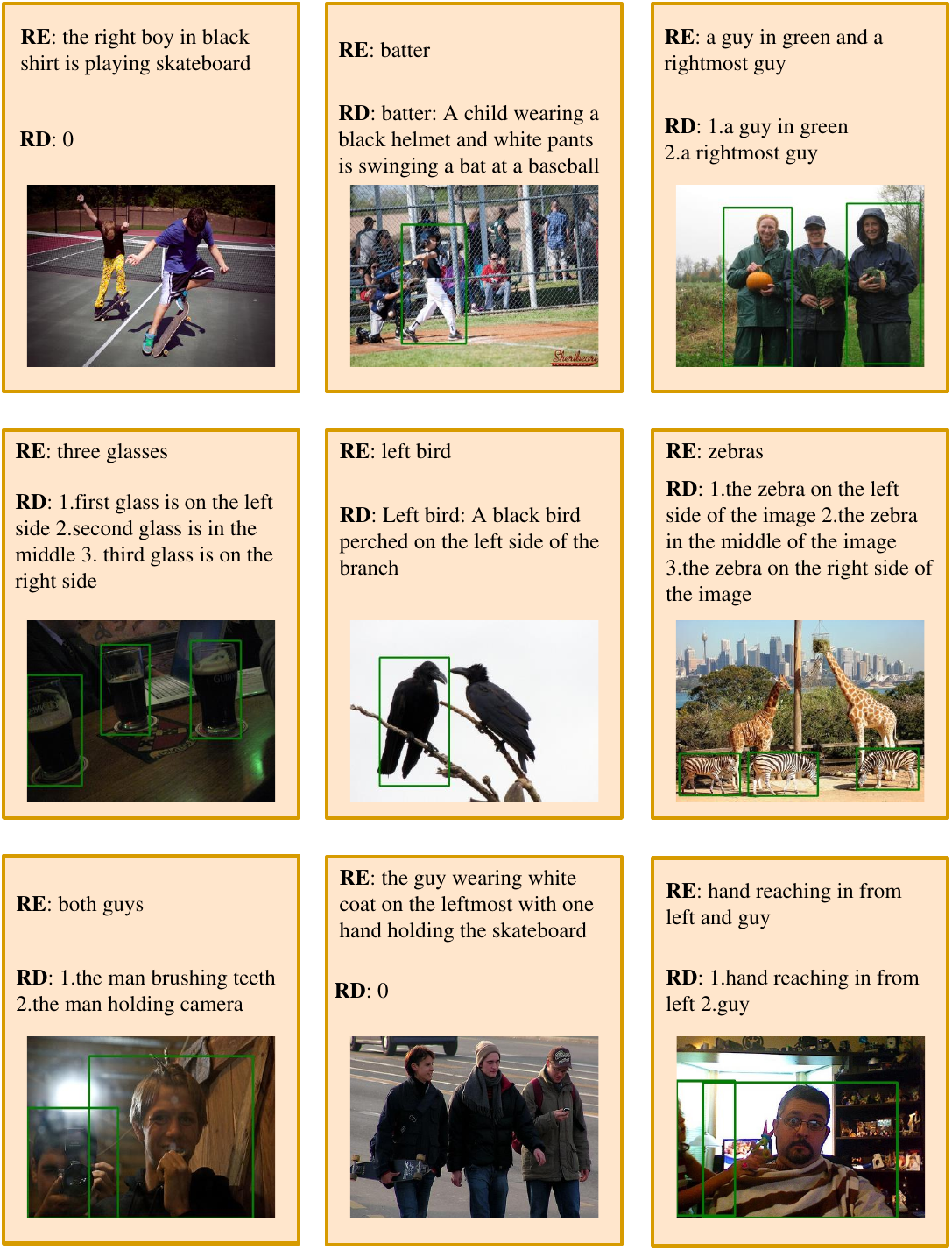}
    \caption{Qualitative visualizations of LIHE. RE denotes the original Referring Expression, and RD denotes the result of Referential Decoupling. The image is the result of  Referent Grounding.}
    \label{visualize0}
\end{figure*}
\begin{figure*}
    \centering
    \setlength{\abovecaptionskip}{6pt}   
    \setlength{\belowcaptionskip}{1pt}  
    \includegraphics[width=1.0\linewidth]{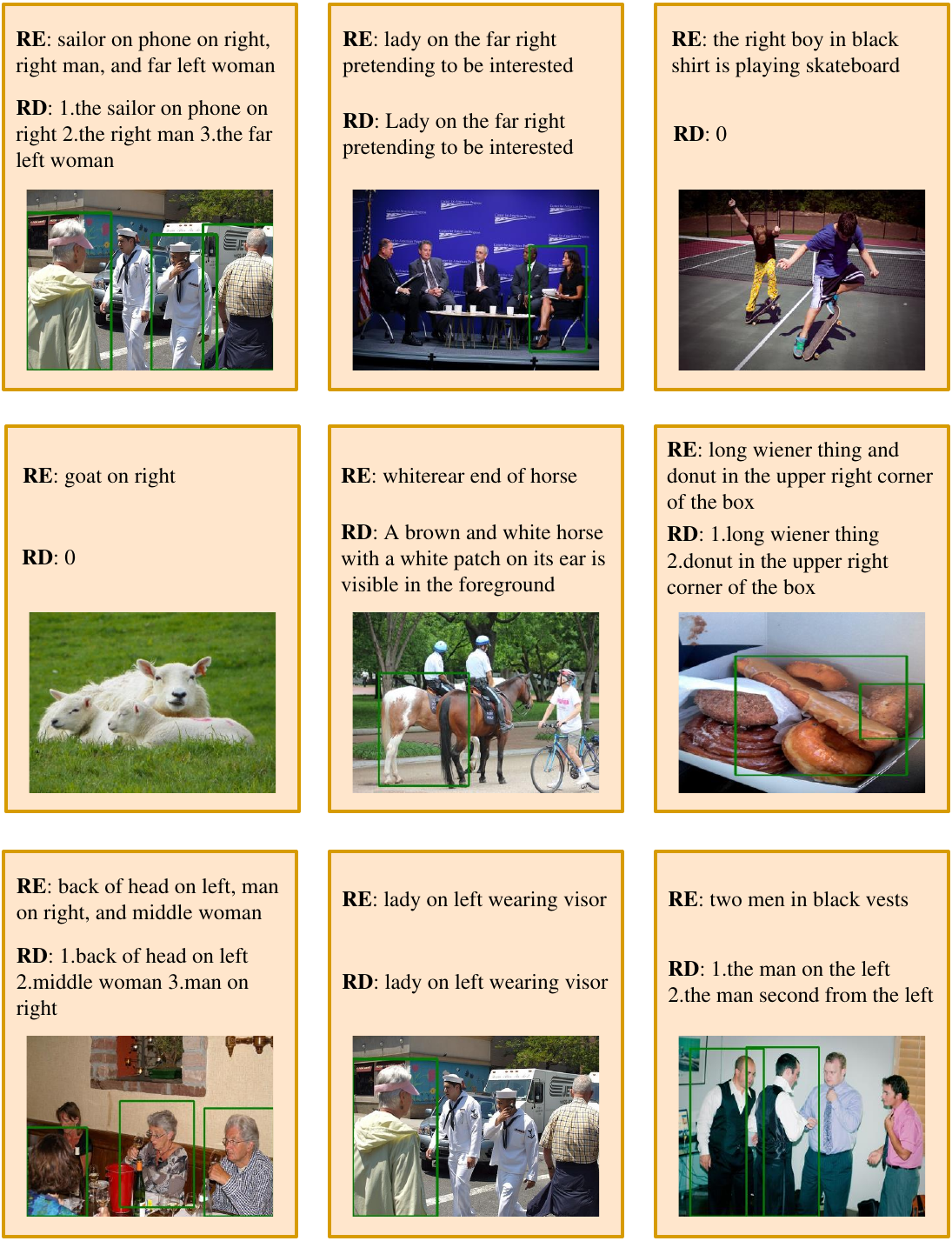}
    \caption{Qualitative visualizations of LIHE. RE denotes the original Referring Expression, and RD denotes the result of Referential Decoupling. The image is the result of  Referent Grounding.}
    
    \label{visualize1}
\end{figure*}
\begin{figure*}
    \centering
    \setlength{\abovecaptionskip}{6pt}   
    \setlength{\belowcaptionskip}{1pt}  
    \includegraphics[width=1.0\linewidth]{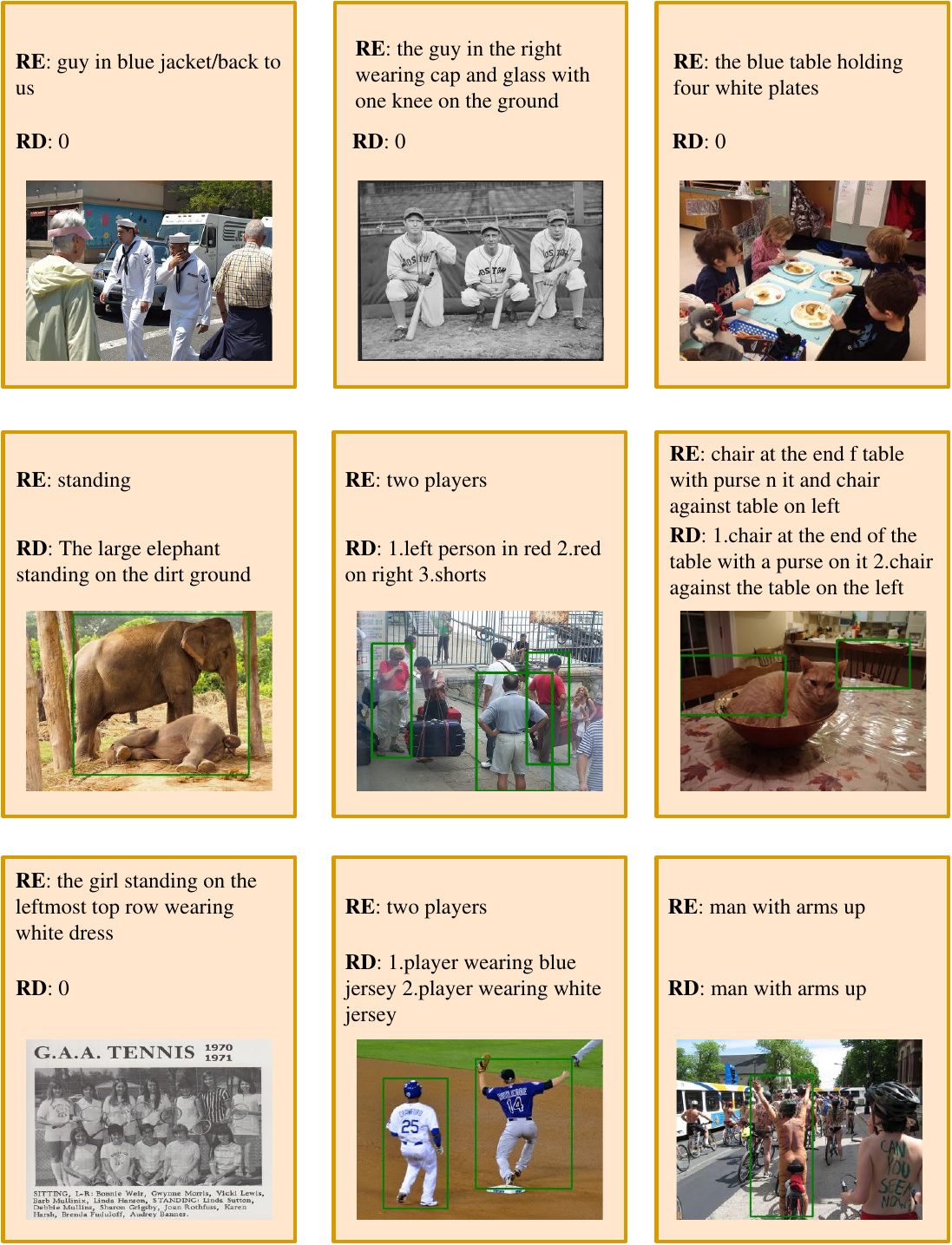}
    \caption{Qualitative visualizations of LIHE. RE denotes the original Referring Expression, and RD denotes the result of Referential Decoupling. The image is the result of  Referent Grounding.}
    \label{visualize2}
\end{figure*}
\begin{figure*}
    \centering
    \setlength{\abovecaptionskip}{6pt}   
    \setlength{\belowcaptionskip}{1pt}  
    \includegraphics[width=1.0\linewidth]{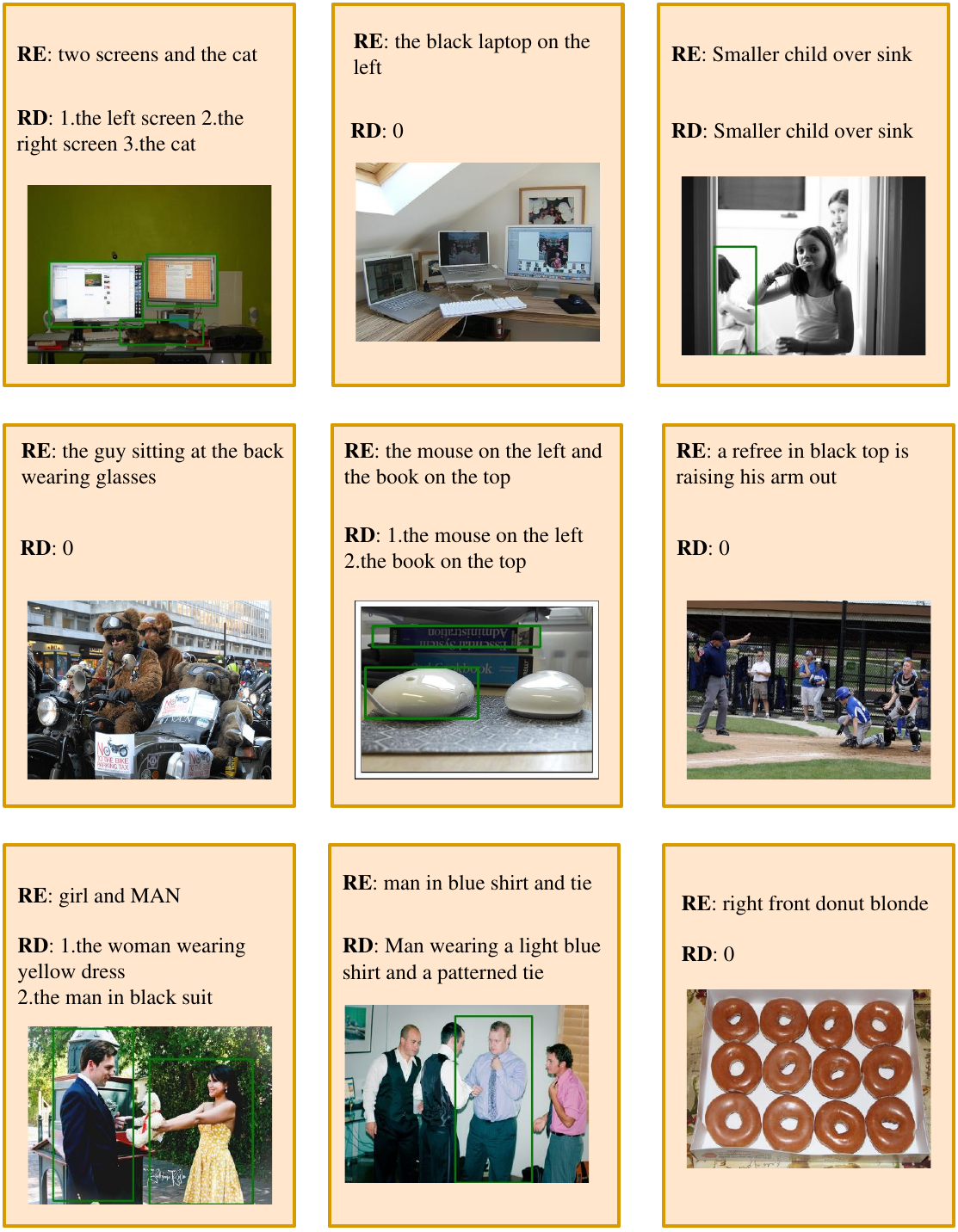}
    \caption{Qualitative visualizations of LIHE. RE denotes the original Referring Expression, and RD denotes the result of Referential Decoupling. The image is the result of  Referent Grounding.}
    \label{visualize3}
\end{figure*}
\begin{figure*}
    \centering
    \setlength{\abovecaptionskip}{6pt}   
    \setlength{\belowcaptionskip}{1pt}  
    \includegraphics[width=1.0\linewidth]{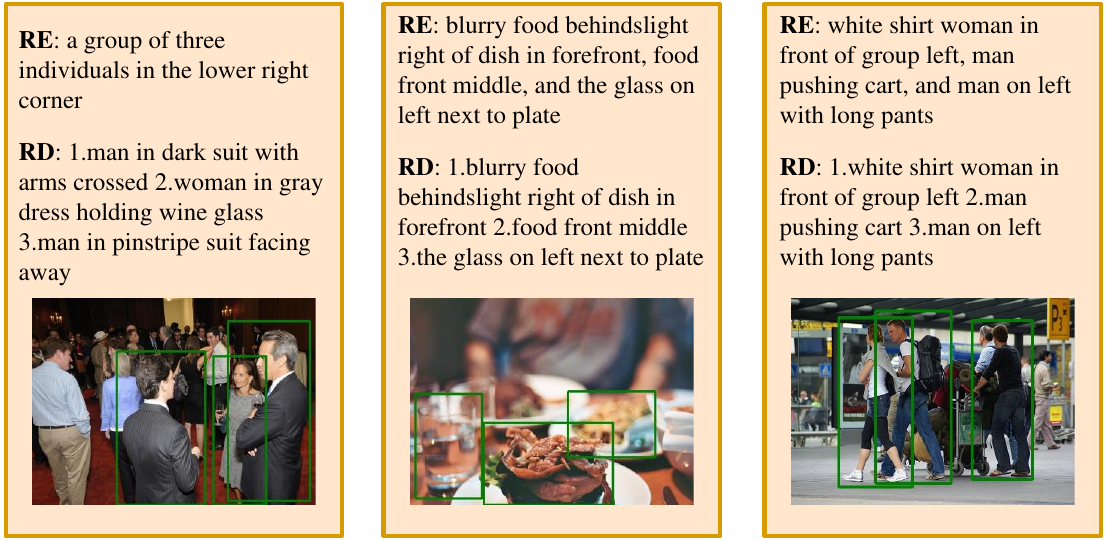}
    \caption{Qualitative visualizations of LIHE. RE denotes the original Referring Expression, and RD denotes the result of Referential Decoupling. The image is the result of  Referent Grounding.}
    \label{visualize4}
\end{figure*}

%% file: iclr2026_conference.bbl
\begin{thebibliography}{77}
\providecommand{\natexlab}[1]{#1}
\providecommand{\url}[1]{\texttt{#1}}
\expandafter\ifx\csname urlstyle\endcsname\relax
  \providecommand{\doi}[1]{doi: #1}\else
  \providecommand{\doi}{doi: \begingroup \urlstyle{rm}\Url}\fi

\bibitem[Atigh et~al.(2022)Atigh, Schoep, Acar, Van~Noord, and Mettes]{hypseg1}
Mina~Ghadimi Atigh, Julian Schoep, Erman Acar, Nanne Van~Noord, and Pascal Mettes.
\newblock Hyperbolic image segmentation.
\newblock In \emph{Proceedings of the IEEE/CVF conference on computer vision and pattern recognition}, pp.\  4453--4462, 2022.

\bibitem[Bai et~al.(2023)Bai, Bai, Yang, Wang, Tan, Wang, Lin, Zhou, and Zhou]{Qwen-VL}
Jinze Bai, Shuai Bai, Shusheng Yang, Shijie Wang, Sinan Tan, Peng Wang, Junyang Lin, Chang Zhou, and Jingren Zhou.
\newblock Qwen-vl: A versatile vision-language model for understanding, localization, text reading, and beyond.
\newblock \emph{arXiv preprint arXiv:2308.12966}, 2023.

\bibitem[Bdeir et~al.(2024)Bdeir, Schwethelm, and Landwehr]{bdeir2024fully}
Ahmad Bdeir, Kristian Schwethelm, and Niels Landwehr.
\newblock Fully hyperbolic convolutional neural networks for computer vision.
\newblock In \emph{International Conference on Learning Representations (ICLR)}, 2024.

\bibitem[Cheng et~al.(2025)Cheng, Liu, He, Ourselin, Tan, and Luo]{cheng2025weakmcn}
Silin Cheng, Yang Liu, Xinwei He, Sebastien Ourselin, Lei Tan, and Gen Luo.
\newblock Weakmcn: Multi-task collaborative network for weakly supervised referring expression comprehension and segmentation.
\newblock In \emph{Proceedings of the Computer Vision and Pattern Recognition Conference}, pp.\  9175--9185, 2025.

\bibitem[Deng et~al.(2021)Deng, Yang, Chen, Zhou, and Li]{deng2021transvg}
Jiajun Deng, Zhengyuan Yang, Tianlang Chen, Wengang Zhou, and Houqiang Li.
\newblock Transvg: End-to-end visual grounding with transformers.
\newblock In \emph{Proceedings of the IEEE/CVF International Conference on Computer Vision}, pp.\  1769--1779, 2021.

\bibitem[Deng et~al.(2023)Deng, Yang, Liu, Chen, Zhou, Zhang, Li, and Ouyang]{deng2023transvg}
Jiajun Deng, Zhengyuan Yang, Daqing Liu, Tianlang Chen, Wengang Zhou, Yanyong Zhang, Houqiang Li, and Wanli Ouyang.
\newblock Transvg++: End-to-end visual grounding with language conditioned vision transformer.
\newblock \emph{IEEE transactions on pattern analysis and machine intelligence}, 45\penalty0 (11):\penalty0 13636--13652, 2023.

\bibitem[Desai et~al.(2023)Desai, Nickel, Rajpurohit, Johnson, and Vedantam]{hypmeru}
Karan Desai, Maximilian Nickel, Tanmay Rajpurohit, Justin Johnson, and Shanmukha~Ramakrishna Vedantam.
\newblock Hyperbolic image-text representations.
\newblock In \emph{International Conference on Machine Learning}, pp.\  7694--7731. PMLR, 2023.

\bibitem[Ding et~al.(2021)Ding, Liu, Wang, and Jiang]{ding2021vision}
Henghui Ding, Chang Liu, Suchen Wang, and Xudong Jiang.
\newblock Vision-language transformer and query generation for referring segmentation.
\newblock In \emph{Proceedings of the IEEE/CVF international conference on computer vision}, pp.\  16321--16330, 2021.

\bibitem[Ermolov et~al.(2022)Ermolov, Mirvakhabova, Khrulkov, Sebe, and Oseledets]{hypvit}
Aleksandr Ermolov, Leyla Mirvakhabova, Valentin Khrulkov, Nicu Sebe, and Ivan Oseledets.
\newblock Hyperbolic vision transformers: Combining improvements in metric learning.
\newblock In \emph{Proceedings of the IEEE/CVF Conference on Computer Vision and Pattern Recognition}, pp.\  7409--7419, 2022.

\bibitem[Ganea et~al.(2018)Ganea, Bécigneul, and Hofmann]{ganea2018hyperbolic}
Octavian-Emanuel Ganea, Gary Bécigneul, and Thomas Hofmann.
\newblock Hyperbolic neural networks.
\newblock In \emph{Advances in Neural Information Processing Systems}, pp.\  5345--5355, 2018.

\bibitem[Gao et~al.(2021)Gao, Wu, Jia, and Harandi]{hypfs}
Zhi Gao, Yuwei Wu, Yunde Jia, and Mehrtash Harandi.
\newblock Curvature generation in curved spaces for few-shot learning.
\newblock In \emph{Proceedings of the IEEE/CVF international conference on computer vision}, pp.\  8691--8700, 2021.

\bibitem[Ge et~al.(2023)Ge, Mishra, Kornblith, Li, and Jacobs]{ge2023hyperbolic}
Songwei Ge, Shlok Mishra, Simon Kornblith, Chun-Liang Li, and David Jacobs.
\newblock Hyperbolic contrastive learning for visual representations beyond objects.
\newblock In \emph{Proceedings of the IEEE/CVF conference on computer vision and pattern recognition}, pp.\  6840--6849, 2023.

\bibitem[Gouk et~al.(2021)Gouk, Frank, Pfahringer, and Cree]{gouk2021regularisation}
Henry Gouk, Eibe Frank, Bernhard Pfahringer, and Michael~J Cree.
\newblock Regularisation of neural networks by enforcing lipschitz continuity.
\newblock \emph{Machine Learning}, 110:\penalty0 393--416, 2021.

\bibitem[Gupta et~al.(2020)Gupta, Vahdat, Chechik, Yang, Kautz, and Hoiem]{gupta2020contrastive}
Tanmay Gupta, Arash Vahdat, Gal Chechik, Xiaodong Yang, Jan Kautz, and Derek Hoiem.
\newblock Contrastive learning for weakly supervised phrase grounding.
\newblock In \emph{European Conference on Computer Vision}, pp.\  752--768, 2020.

\bibitem[He et~al.(2023)He, Ding, Liu, and Jiang]{he2023grec}
Shuting He, Henghui Ding, Chang Liu, and Xudong Jiang.
\newblock Grec: Generalized referring expression comprehension.
\newblock \emph{arXiv preprint arXiv:2308.16182}, 2023.

\bibitem[Hemanthage et~al.(2024)Hemanthage, Bilen, Bartie, Dondrup, and Lemon]{hemanthage2024recantformer}
Bhathiya Hemanthage, Hakan Bilen, Phil Bartie, Christian Dondrup, and Oliver Lemon.
\newblock Recantformer: Referring expression comprehension with varying numbers of targets.
\newblock In \emph{Proceedings of the 2024 Conference on Empirical Methods in Natural Language Processing}, pp.\  21784--21798, 2024.

\bibitem[Ho et~al.(2022)Ho, Appalaraju, Jasani, Manmatha, and Vasconcelos]{ho2022yoro}
Chih-Hui Ho, Srikar Appalaraju, Bhavan Jasani, R~Manmatha, and Nuno Vasconcelos.
\newblock Yoro—lightweight end to end visual grounding.
\newblock In \emph{European Conference on Computer Vision}, pp.\  3--23, 2022.

\bibitem[Hong et~al.(2019)Hong, Liu, Mo, He, and Zhang]{hong2019ltrs}
Richang Hong, Daqing Liu, Xiaoyu Mo, Xiangnan He, and Hanwang Zhang.
\newblock Learning to compose and reason with language tree structures for visual grounding.
\newblock \emph{IEEE Transactions on Pattern Analysis and Machine Intelligence}, 44\penalty0 (2):\penalty0 684--696, 2019.

\bibitem[Hu et~al.(2023)Hu, Wang, Shao, Xie, Li, Han, and Luo]{hu2023beyond}
Yutao Hu, Qixiong Wang, Wenqi Shao, Enze Xie, Zhenguo Li, Jungong Han, and Ping Luo.
\newblock Beyond one-to-one: Rethinking the referring image segmentation.
\newblock In \emph{Proceedings of the IEEE/CVF International Conference on Computer Vision}, pp.\  4067--4077, 2023.

\bibitem[Huang et~al.(2021{\natexlab{a}})Huang, Lian, Luo, and Gao]{huang2021look}
Binbin Huang, Dongze Lian, Weixin Luo, and Shenghua Gao.
\newblock Look before you leap: Learning landmark features for one-stage visual grounding.
\newblock In \emph{Proceedings of the IEEE/CVF Conference on Computer Vision and Pattern Recognition}, pp.\  16888--16897, 2021{\natexlab{a}}.

\bibitem[Huang et~al.(2021{\natexlab{b}})Huang, Yi, Zhao, and Jiang]{huang2021towards}
Weiran Huang, Mingyang Yi, Xuyang Zhao, and Zihao Jiang.
\newblock Towards the generalization of contrastive self-supervised learning.
\newblock \emph{arXiv preprint arXiv:2111.00743}, 2021{\natexlab{b}}.

\bibitem[Jiang et~al.(2024)Jiang, Shi, Jiang, Feng, Lu, and Xu]{diffusaliency}
Yunhan Jiang, Xianglong Shi, Xiaoheng Jiang, Jian Feng, Yang Lu, and Mingliang Xu.
\newblock Diffusaliency: Synthesizing multi-object images with masks for semantic segmentation using diffusion and saliency detection.
\newblock In \emph{Chinese Conference on Pattern Recognition and Computer Vision (PRCV)}, pp.\  74--88. Springer, 2024.

\bibitem[Jin et~al.(2023)Jin, Luo, Zhou, Sun, Jiang, Shu, and Ji]{jin2023refclip}
Lei Jin, Gen Luo, Yiyi Zhou, Xiaoshuai Sun, Guannan Jiang, Annan Shu, and Rongrong Ji.
\newblock Refclip: A universal teacher for weakly supervised referring expression comprehension.
\newblock In \emph{Proceedings of the IEEE/CVF conference on computer vision and pattern recognition}, pp.\  2681--2690, 2023.

\bibitem[Kamath et~al.(2021)Kamath, Singh, LeCun, Synnaeve, Misra, and Carion]{kamath2021mdetr}
Aishwarya Kamath, Mannat Singh, Yann LeCun, Gabriel Synnaeve, Ishan Misra, and Nicolas Carion.
\newblock Mdetr-modulated detection for end-to-end multi-modal understanding.
\newblock In \emph{Proceedings of the IEEE/CVF international conference on computer vision}, pp.\  1780--1790, 2021.

\bibitem[Khrulkov et~al.(2020)Khrulkov, Mirvakhabova, Ustinova, Oseledets, and Lempitsky]{khrulkov2020hyperbolic}
Valentin Khrulkov, Leyla Mirvakhabova, Evgeniya Ustinova, Ivan Oseledets, and Victor Lempitsky.
\newblock Hyperbolic image embeddings.
\newblock In \emph{Proceedings of the IEEE/CVF Conference on Computer Vision and Pattern Recognition}, pp.\  6418--6428, 2020.

\bibitem[Kim et~al.(2023)Kim, Jeong, and Kwak]{hier}
Sungyeon Kim, Boseung Jeong, and Suha Kwak.
\newblock Hier: Metric learning beyond class labels via hierarchical regularization.
\newblock In \emph{Proceedings of the IEEE/CVF Conference on Computer Vision and Pattern Recognition}, pp.\  19903--19912, 2023.

\bibitem[Kong et~al.(2024)Kong, Chen, Cai, and Modolo]{kong2024hyperbolic}
Fanjie Kong, Yanbei Chen, Jiarui Cai, and Davide Modolo.
\newblock Hyperbolic learning with synthetic captions for open-world detection.
\newblock In \emph{Proceedings of the IEEE/CVF Conference on Computer Vision and Pattern Recognition}, pp.\  16762--16771, 2024.

\bibitem[Kwon et~al.(2024)Kwon, Jang, Kim, Kim, and Sohn]{kwon2024improving}
Hyeongjun Kwon, Jinhyun Jang, Jin Kim, Kwonyoung Kim, and Kwanghoon Sohn.
\newblock Improving visual recognition with hyperbolical visual hierarchy mapping.
\newblock In \emph{Proceedings of the IEEE/CVF Conference on Computer Vision and Pattern Recognition}, pp.\  17364--17374, 2024.

\bibitem[Lei et~al.(2023)Lei, Yang, Ying, and Zhou]{lei2023generalization}
Yunwen Lei, Tianbao Yang, Yiming Ying, and Ding-Xuan Zhou.
\newblock Generalization analysis for contrastive representation learning.
\newblock In \emph{International Conference on Machine Learning}, pp.\  19200--19227. PMLR, 2023.

\bibitem[Li et~al.(2024)Li, Chen, Xu, and Hu]{li2024hyperbolic}
Huimin Li, Zhentao Chen, Yunhao Xu, and Junlin Hu.
\newblock Hyperbolic anomaly detection.
\newblock In \emph{Proceedings of the IEEE/CVF Conference on Computer Vision and Pattern Recognition}, pp.\  17511--17520, 2024.

\bibitem[Liang et~al.(2025)Liang, Yu, Mu, Zhuang, Hu, Yang, Ye, Lu, Chen, and Hu]{liang2025rbd}
Xiaoyu Liang, Jiayuan Yu, Lianrui Mu, Jiedong Zhuang, Jiaqi Hu, Yuchen Yang, Jiangnan Ye, Lu~Lu, Jian Chen, and Haoji Hu.
\newblock Mitigating hallucination in visual-language models via re-balancing contrastive decoding.
\newblock In \emph{Pattern Recognition and Computer Vision (PRCV 2024)}, volume 15035 of \emph{Lecture Notes in Computer Science}, pp.\  482--496, Singapore, 2025. Springer.
\newblock \doi{10.1007/978-981-97-8620-6_33}.
\newblock URL \url{https://link.springer.com/chapter/10.1007/978-981-97-8620-6_33}.

\bibitem[Liao et~al.(2020)Liao, Liu, Li, Wang, Chen, Qian, and Li]{liao2020cmcf}
Yue Liao, Si~Liu, Guanbin Li, Fei Wang, Yanjie Chen, Chen Qian, and Bo~Li.
\newblock A real-time cross-modality correlation filtering method for referring expression comprehension.
\newblock In \emph{Proceedings of the IEEE/CVF Conference on Computer Vision and Pattern Recognition}, pp.\  10880--10889, 2020.

\bibitem[Lin et~al.(2014)Lin, Maire, Belongie, Hays, Perona, Ramanan, Doll{\'a}r, and Zitnick]{lin2014microsoft}
Tsung-Yi Lin, Michael Maire, Serge Belongie, James Hays, Pietro Perona, Deva Ramanan, Piotr Doll{\'a}r, and C~Lawrence Zitnick.
\newblock Microsoft coco: Common objects in context.
\newblock In \emph{Computer Vision--ECCV 2014: 13th European Conference, Zurich, Switzerland, September 6-12, 2014, Proceedings, Part V 13}, pp.\  740--755. Springer, 2014.

\bibitem[Liu et~al.(2023)Liu, Ding, and Jiang]{liu2023gres}
Chang Liu, Henghui Ding, and Xudong Jiang.
\newblock Gres: Generalized referring expression segmentation.
\newblock In \emph{Proceedings of the IEEE/CVF conference on computer vision and pattern recognition}, pp.\  23592--23601, 2023.

\bibitem[Liu et~al.(2019{\natexlab{a}})Liu, Zhang, Wu, and Zha]{liu2019nmtree}
Daqing Liu, Hanwang Zhang, Feng Wu, and Zheng-Jun Zha.
\newblock Learning to assemble neural module tree networks for visual grounding.
\newblock In \emph{Proceedings of the IEEE/CVF International Conference on Computer Vision}, pp.\  4673--4682, 2019{\natexlab{a}}.

\bibitem[Liu et~al.(2019{\natexlab{b}})Liu, Wang, Shao, Wang, and Li]{liu2019erase}
Xihui Liu, Zihao Wang, Jing Shao, Xiaogang Wang, and Hongsheng Li.
\newblock Improving referring expression grounding with cross-modal attention-guided erasing.
\newblock In \emph{Proceedings of the IEEE/CVF Conference on Computer Vision and Pattern Recognition}, pp.\  1950--1959, 2019{\natexlab{b}}.

\bibitem[Liu et~al.(2019{\natexlab{c}})Liu, Li, Wang, Zha, Meng, and Huang]{liu2019adaptive}
Xuejing Liu, Liang Li, Shuhui Wang, Zheng-Jun Zha, Dechao Meng, and Qingming Huang.
\newblock Adaptive reconstruction network for weakly supervised referring expression grounding.
\newblock In \emph{Proceedings of the IEEE/CVF International Conference on Computer Vision}, pp.\  2611--2620, 2019{\natexlab{c}}.

\bibitem[Liu et~al.(2019{\natexlab{d}})Liu, Li, Wang, Zha, Su, and Huang]{liu2019kpnet}
Xuejing Liu, Liang Li, Shuhui Wang, Zheng-Jun Zha, Li~Su, and Qingming Huang.
\newblock Knowledge-guided pairwise reconstruction network for weakly supervised referring expression grounding.
\newblock In \emph{Proceedings of the 27th ACM International Conference on Multimedia}, pp.\  539--547, 2019{\natexlab{d}}.

\bibitem[Liu et~al.(2021)Liu, Wan, Ma, and He]{liu2021rir}
Yongfei Liu, Bo~Wan, Lin Ma, and Xuming He.
\newblock Relation-aware instance refinement for weakly supervised visual grounding.
\newblock In \emph{Proceedings of the IEEE/CVF Conference on Computer Vision and Pattern Recognition}, pp.\  5612--5621, 2021.

\bibitem[Luo et~al.(2020)Luo, Zhou, Sun, Cao, Wu, Deng, and Ji]{luo2020multi}
Gen Luo, Yiyi Zhou, Xiaoshuai Sun, Liujuan Cao, Chenglin Wu, Cheng Deng, and Rongrong Ji.
\newblock Multi-task collaborative network for joint referring expression comprehension and segmentation.
\newblock In \emph{Proceedings of the IEEE/CVF Conference on computer vision and pattern recognition}, pp.\  10034--10043, 2020.

\bibitem[Luo et~al.(2024)Luo, Ji, Chen, Zhang, Ren, and Luo]{luo2024apl}
Yaxin Luo, Jiayi Ji, Xiaofu Chen, Yuxin Zhang, Tianhe Ren, and Gen Luo.
\newblock Apl: Anchor-based prompt learning for one-stage weakly supervised referring expression comprehension.
\newblock In \emph{European Conference on Computer Vision}, pp.\  198--215. Springer, 2024.

\bibitem[Mao et~al.(2016)Mao, Huang, Toshev, Camburu, Yuille, and Murphy]{mao2016generation}
Junhua Mao, Jonathan Huang, Alexander Toshev, Oana Camburu, Alan~L Yuille, and Kevin Murphy.
\newblock Generation and comprehension of unambiguous object descriptions.
\newblock In \emph{Proceedings of the IEEE conference on computer vision and pattern recognition}, pp.\  11--20, 2016.

\bibitem[McKenna et~al.(2023)McKenna, Li, Cheng, Hosseini, Johnson, and Steedman]{mckenna-etal-2023-sources}
Nick McKenna, Tianyi Li, Liang Cheng, Mohammad Hosseini, Mark Johnson, and Mark Steedman.
\newblock Sources of hallucination by large language models on inference tasks.
\newblock In Houda Bouamor, Juan Pino, and Kalika Bali (eds.), \emph{Findings of the Association for Computational Linguistics: EMNLP 2023}, pp.\  2758--2774, Singapore, December 2023. Association for Computational Linguistics.
\newblock \doi{10.18653/v1/2023.findings-emnlp.182}.
\newblock URL \url{https://aclanthology.org/2023.findings-emnlp.182/}.

\bibitem[Nagaraja et~al.(2016)Nagaraja, Morariu, and Davis]{nagaraja2016modeling}
Varun~K Nagaraja, Vlad~I Morariu, and Larry~S Davis.
\newblock Modeling context between objects for referring expression understanding.
\newblock In \emph{Computer Vision--ECCV 2016: 14th European Conference, Amsterdam, The Netherlands, October 11--14, 2016, Proceedings, Part IV 14}, pp.\  792--807. Springer, 2016.

\bibitem[Nickel \& Kiela(2017{\natexlab{a}})Nickel and Kiela]{nickel2017poincare}
Maximilian Nickel and Douwe Kiela.
\newblock Poincaré embeddings for learning hierarchical representations.
\newblock In \emph{Advances in Neural Information Processing Systems}, pp.\  6338--6347, 2017{\natexlab{a}}.

\bibitem[Nickel \& Kiela(2017{\natexlab{b}})Nickel and Kiela]{hyppoincare}
Maximillian Nickel and Douwe Kiela.
\newblock Poincar{\'e} embeddings for learning hierarchical representations.
\newblock \emph{Advances in neural information processing systems}, 30, 2017{\natexlab{b}}.

\bibitem[Nickel \& Kiela(2018)Nickel and Kiela]{hyplorentz}
Maximillian Nickel and Douwe Kiela.
\newblock Learning continuous hierarchies in the lorentz model of hyperbolic geometry.
\newblock In \emph{International conference on machine learning}, pp.\  3779--3788. PMLR, 2018.

\bibitem[Niu et~al.(2019)Niu, Zhang, Lu, and Chang]{niu2019variational}
Yulei Niu, Hanwang Zhang, Zhiwu Lu, and Shih-Fu Chang.
\newblock Variational context: Exploiting visual and textual context for grounding referring expressions.
\newblock \emph{IEEE transactions on pattern analysis and machine intelligence}, 43\penalty0 (1):\penalty0 347--359, 2019.

\bibitem[Oord et~al.(2018)Oord, Li, and Vinyals]{oord2018representation}
Aaron van~den Oord, Yazhe Li, and Oriol Vinyals.
\newblock Representation learning with contrastive predictive coding.
\newblock \emph{arXiv preprint arXiv:1807.03748}, 2018.

\bibitem[Pramanick et~al.(2024)Pramanick, Han, Hou, Nag, Lim, Ballas, Wang, Chellappa, and Almahairi]{pramanick2024jack}
Shraman Pramanick, Guangxing Han, Rui Hou, Sayan Nag, Ser-Nam Lim, Nicolas Ballas, Qifan Wang, Rama Chellappa, and Amjad Almahairi.
\newblock Jack of all tasks master of many: Designing general-purpose coarse-to-fine vision-language model.
\newblock In \emph{Proceedings of the IEEE/CVF Conference on Computer Vision and Pattern Recognition}, pp.\  14076--14088, 2024.

\bibitem[Ramasinghe et~al.(2024)Ramasinghe, Shevchenko, Avraham, and Thalaiyasingam]{ramasinghe2024accept}
Sameera Ramasinghe, Violetta Shevchenko, Gil Avraham, and Ajanthan Thalaiyasingam.
\newblock Accept the modality gap: An exploration in the hyperbolic space.
\newblock In \emph{Proceedings of the IEEE/CVF Conference on Computer Vision and Pattern Recognition}, pp.\  27263--27272, 2024.

\bibitem[Redmon \& Farhadi(2018)Redmon and Farhadi]{redmon2018yolov3}
Joseph Redmon and Ali Farhadi.
\newblock Yolov3: An incremental improvement.
\newblock \emph{arXiv preprint arXiv:1804.02767}, 2018.

\bibitem[Rohrbach et~al.(2018)Rohrbach, Hendricks, Burns, Darrell, and Saenko]{rohrbach-etal-2018-object}
Anna Rohrbach, Lisa~Anne Hendricks, Kaylee Burns, Trevor Darrell, and Kate Saenko.
\newblock Object hallucination in image captioning.
\newblock In Ellen Riloff, David Chiang, Julia Hockenmaier, and Jun{'}ichi Tsujii (eds.), \emph{Proceedings of the 2018 Conference on Empirical Methods in Natural Language Processing}, pp.\  4035--4045, Brussels, Belgium, October-November 2018. Association for Computational Linguistics.
\newblock \doi{10.18653/v1/D18-1437}.
\newblock URL \url{https://aclanthology.org/D18-1437/}.

\bibitem[Shi et~al.(2025)Shi, Jiang, Jiang, Xu, and Liu]{crossdiff}
Xianglong Shi, Yunhan Jiang, Xiaoheng Jiang, Mingling Xu, and Yang Liu.
\newblock Crossdiff: Diffusion probabilistic model with cross-conditional encoder-decoder for crack segmentation.
\newblock \emph{arXiv preprint arXiv:2501.12860}, 2025.

\bibitem[Sun et~al.(2021{\natexlab{a}})Sun, Xiao, and Lim]{sun2021iterative}
Mingjie Sun, Jimin Xiao, and Eng~Gee Lim.
\newblock Iterative shrinking for referring expression grounding using deep reinforcement learning.
\newblock In \emph{Proceedings of the IEEE/CVF Conference on Computer Vision and Pattern Recognition}, pp.\  14060--14069, 2021{\natexlab{a}}.

\bibitem[Sun et~al.(2021{\natexlab{b}})Sun, Xiao, Lim, Liu, and Goulermas]{sun2021discriminative}
Mingjie Sun, Jimin Xiao, Eng~Gee Lim, Si~Liu, and John~Y Goulermas.
\newblock Discriminative triad matching and reconstruction for weakly referring expression grounding.
\newblock \emph{IEEE transactions on pattern analysis and machine intelligence}, 43\penalty0 (11):\penalty0 4189--4195, 2021{\natexlab{b}}.

\bibitem[Team(2025)]{qwen2.5-VL}
Qwen Team.
\newblock Qwen2.5-vl, January 2025.
\newblock URL \url{https://qwenlm.github.io/blog/qwen2.5-vl/}.

\bibitem[Tifrea et~al.(2018)Tifrea, B{\'e}cigneul, and Ganea]{hyptxt1s}
Alexandru Tifrea, Gary B{\'e}cigneul, and Octavian-Eugen Ganea.
\newblock Poincar$\backslash$'e glove: Hyperbolic word embeddings.
\newblock \emph{arXiv preprint arXiv:1810.06546}, 2018.

\bibitem[Wang et~al.(2021)Wang, Huang, Li, Xu, Yang, and Yu]{wang2021ckd}
Liwei Wang, Jing Huang, Yin Li, Kun Xu, Zhengyuan Yang, and Dong Yu.
\newblock Improving weakly supervised visual grounding by contrastive knowledge distillation.
\newblock In \emph{Proceedings of the IEEE/CVF Conference on Computer Vision and Pattern Recognition}, pp.\  14090--14100, 2021.

\bibitem[Wang et~al.(2024)Wang, Bai, Tan, Wang, Fan, Bai, Chen, Liu, Wang, Ge, Fan, Dang, Du, Ren, Men, Liu, Zhou, Zhou, and Lin]{Qwen2VL}
Peng Wang, Shuai Bai, Sinan Tan, Shijie Wang, Zhihao Fan, Jinze Bai, Keqin Chen, Xuejing Liu, Jialin Wang, Wenbin Ge, Yang Fan, Kai Dang, Mengfei Du, Xuancheng Ren, Rui Men, Dayiheng Liu, Chang Zhou, Jingren Zhou, and Junyang Lin.
\newblock Qwen2-vl: Enhancing vision-language model's perception of the world at any resolution.
\newblock \emph{arXiv preprint arXiv:2409.12191}, 2024.

\bibitem[Wang et~al.(2025)Wang, Ding, He, Jiang, Wei, and Liu]{wang2025hierarchical}
Yaxian Wang, Henghui Ding, Shuting He, Xudong Jiang, Bifan Wei, and Jun Liu.
\newblock Hierarchical alignment-enhanced adaptive grounding network for generalized referring expression comprehension.
\newblock \emph{arXiv preprint arXiv:2501.01416}, 2025.

\bibitem[Weng et~al.(2021)Weng, Ogut, Limonchik, and Yeung]{hypseg2}
Zhenzhen Weng, Mehmet~Giray Ogut, Shai Limonchik, and Serena Yeung.
\newblock Unsupervised discovery of the long-tail in instance segmentation using hierarchical self-supervision.
\newblock In \emph{Proceedings of the IEEE/CVF conference on computer vision and pattern recognition}, pp.\  2603--2612, 2021.

\bibitem[Wu et~al.(2024)Wu, Liu, Ji, Ma, Wang, Luo, Ding, Sun, and Ji]{wu20243d}
Changli Wu, Yihang Liu, Jiayi Ji, Yiwei Ma, Haowei Wang, Gen Luo, Henghui Ding, Xiaoshuai Sun, and Rongrong Ji.
\newblock 3d-gres: Generalized 3d referring expression segmentation.
\newblock In \emph{Proceedings of the 32nd ACM International Conference on Multimedia}, pp.\  7852--7861, 2024.

\bibitem[Yan et~al.(2023)Yan, Jiang, Wu, Wang, Luo, Yuan, and Lu]{yan2023universal}
Bin Yan, Yi~Jiang, Jiannan Wu, Dong Wang, Ping Luo, Zehuan Yuan, and Huchuan Lu.
\newblock Universal instance perception as object discovery and retrieval.
\newblock In \emph{Proceedings of the IEEE/CVF Conference on Computer Vision and Pattern Recognition}, pp.\  15325--15336, 2023.

\bibitem[Yang et~al.(2019)Yang, Gong, Wang, Huang, Yu, and Luo]{yang2019fast}
Zhengyuan Yang, Boqing Gong, Liwei Wang, Wenbing Huang, Dong Yu, and Jiebo Luo.
\newblock A fast and accurate one-stage approach to visual grounding.
\newblock In \emph{Proceedings of the IEEE/CVF International Conference on Computer Vision}, pp.\  4683--4693, 2019.

\bibitem[Yang et~al.(2020)Yang, Chen, Wang, and Luo]{yang2020recursive}
Zhengyuan Yang, Tianlang Chen, Liwei Wang, and Jiebo Luo.
\newblock Improving one-stage visual grounding by recursive sub-query construction.
\newblock In \emph{Computer Vision—ECCV 2020}, pp.\  387--404, 2020.

\bibitem[You et~al.(2023)You, Zhang, Gan, Du, Zhang, Wang, Cao, Chang, and Yang]{you2023ferret}
Haoxuan You, Haotian Zhang, Zhe Gan, Xianzhi Du, Bowen Zhang, Zirui Wang, Liangliang Cao, Shih-Fu Chang, and Yinfei Yang.
\newblock Ferret: Refer and ground anything anywhere at any granularity.
\newblock \emph{arXiv preprint arXiv:2310.07704}, 2023.

\bibitem[Yu et~al.(2016)Yu, Poirson, Yang, Berg, and Berg]{yu2016modeling}
Licheng Yu, Patrick Poirson, Shan Yang, Alexander~C Berg, and Tamara~L Berg.
\newblock Modeling context in referring expressions.
\newblock In \emph{Computer Vision--ECCV 2016: 14th European Conference, Amsterdam, The Netherlands, October 11-14, 2016, Proceedings, Part II 14}, pp.\  69--85. Springer, 2016.

\bibitem[Zhang et~al.(2018)Zhang, Niu, and Chang]{zhang2018vc}
Hanwang Zhang, Yulei Niu, and Shih-Fu Chang.
\newblock Grounding referring expressions in images by variational context.
\newblock In \emph{Proceedings of the IEEE Conference on Computer Vision and Pattern Recognition}, pp.\  4158--4166, 2018.

\bibitem[Zhang et~al.(2024)Zhang, Li, Liu, Yu, Fung, Li, Li, and Ji]{zhang2024knowledge}
Yuji Zhang, Sha Li, Jiateng Liu, Pengfei Yu, Yi~R. Fung, Jing Li, Manling Li, and Heng Ji.
\newblock Knowledge overshadowing causes amalgamated hallucination in large language models, 2024.
\newblock URL \url{https://arxiv.org/abs/2407.08039}.

\bibitem[Zhang et~al.(2020)Zhang, Zhao, Lin, He, et~al.]{zhang2020counterfactual}
Zhu Zhang, Zhou Zhao, Zhijie Lin, Xiuqiang He, et~al.
\newblock Counterfactual contrastive learning for weakly-supervised vision-language grounding.
\newblock \emph{Advances in Neural Information Processing Systems}, 33:\penalty0 18123--18134, 2020.

\bibitem[Zhao et~al.(2018)Zhao, Li, Zhao, and Feng]{zhao2018wa}
Fang Zhao, Jianshu Li, Jian Zhao, and Jiashi Feng.
\newblock Weakly supervised phrase localization with multi-scale anchored transformer network.
\newblock In \emph{Proceedings of the IEEE Conference on Computer Vision and Pattern Recognition}, pp.\  5696--5705, 2018.

\bibitem[Zhao et~al.(2022)Zhao, Zhou, and Ong]{zhao2022word2pix}
Heng Zhao, Joey~Tianyi Zhou, and Yew-Soon Ong.
\newblock Word2pix: Word to pixel cross-attention transformer in visual grounding.
\newblock \emph{IEEE Transactions on Neural Networks and Learning Systems}, 35\penalty0 (2):\penalty0 1523--1533, 2022.

\bibitem[Zhou et~al.(2021{\natexlab{a}})Zhou, Ji, Luo, Sun, Su, Ding, Lin, and Tian]{zhou2021ginet}
Yiyi Zhou, Rongrong Ji, Gen Luo, Xiaoshuai Sun, Jinsong Su, Xinghao Ding, Chia-Wen Lin, and Qi~Tian.
\newblock A real-time global inference network for one-stage referring expression comprehension.
\newblock \emph{IEEE Transactions on Neural Networks and Learning Systems}, 34\penalty0 (1):\penalty0 134--143, 2021{\natexlab{a}}.

\bibitem[Zhou et~al.(2021{\natexlab{b}})Zhou, Ren, Zhu, Sun, Liu, Ding, Xu, and Ji]{zhou2021trar}
Yiyi Zhou, Tianhe Ren, Chaoyang Zhu, Xiaoshuai Sun, Jianzhuang Liu, Xinghao Ding, Mingliang Xu, and Rongrong Ji.
\newblock Trar: Routing the attention spans in transformer for visual question answering.
\newblock In \emph{Proceedings of the IEEE/CVF International Conference on Computer Vision}, pp.\  2074--2084, 2021{\natexlab{b}}.

\bibitem[Zhu et~al.(2022)Zhu, Zhou, Shen, Luo, Pan, Lin, Chen, Cao, Sun, and Ji]{zhu2022seqtr}
Chaoyang Zhu, Yiyi Zhou, Yunhang Shen, Gen Luo, Xingjia Pan, Mingbao Lin, Chao Chen, Liujuan Cao, Xiaoshuai Sun, and Rongrong Ji.
\newblock Seqtr: A simple yet universal network for visual grounding.
\newblock In \emph{European Conference on Computer Vision}, pp.\  598--615. Springer, 2022.

\bibitem[Zhu et~al.(2020)Zhu, Zhou, Xiao, Jiang, Chen, and Liu]{hyptxt2s}
Yudong Zhu, Di~Zhou, Jinghui Xiao, Xin Jiang, Xiao Chen, and Qun Liu.
\newblock Hypertext: Endowing fasttext with hyperbolic geometry.
\newblock \emph{arXiv preprint arXiv:2010.16143}, 2020.

\end{thebibliography}
